\documentclass{article}

\usepackage{microtype}
\usepackage{graphicx}
\usepackage{subfigure}
\usepackage{booktabs} % for professional tables
\usepackage{mathtools}
\usepackage{nicefrac}       % compact symbols for 1/2, etc.
\usepackage{hyperref}
\usepackage{mathrsfs}
\usepackage{enumitem}
\usepackage{arydshln}
\usepackage[accepted]{icml2025}

\usepackage{amsmath}
\usepackage{amssymb}
\usepackage{mathtools}
\usepackage{amsthm}

\usepackage[capitalize,noabbrev]{cleveref}

\theoremstyle{plain}
\newtheorem{theorem}{Theorem}[section]
\newtheorem{proposition}[theorem]{Proposition}

\theoremstyle{definition}
\newtheorem{definition}[theorem]{Definition}

\theoremstyle{remark}

\usepackage[textsize=tiny]{todonotes}

\usepackage{ragged2e} 
\usepackage{tabu}                     % 表格插入
\usepackage{multicol}                 % 合并多列
\usepackage{multirow}                % 合并多行
\usepackage{float}                    % 图片浮动
\usepackage{makecell}                 % 三线表-竖线
\usepackage{booktabs}                 % 三线表-短细横线

\usepackage{tcolorbox}

\icmltitlerunning{Purity Law for Generalizable Neural TSP Solvers}

\begin{document}

\twocolumn[
\icmltitle{Purity Law for Generalizable Neural TSP Solvers}
% Towards Generalizable Neural Solvers for the Traveling Salesman Problem Based on the Purity Law
% LOPUS: Generalizable Neural Solvers Based on Local Purity Property for the Traveling Salesman Problem\\
% Towards Scale and Distributionally Generalizable Neural Solvers Based on Local Purity Property for the Traveling Salesman Problem
% Local Purity: An Scale- and Distribution-invariant Property for the Traveling Salesman Problem towards Better Learning Generalization

% It is OKAY to include author information, even for blind
% submissions: the style file will automatically remove it for you
% unless you've provided the [accepted] option to the icml2025
% package.

% List of affiliations: The first argument should be a (short)
% identifier you will use later to specify author affiliations
% Academic affiliations should list Department, University, City, Region, Country
% Industry affiliations should list Company, City, Region, Country

% You can specify symbols, otherwise they are numbered in order.
% Ideally, you should not use this facility. Affiliations will be numbered
% in order of appearance and this is the preferred way.
\icmlsetsymbol{equal}{*}

\begin{icmlauthorlist}
\icmlauthor{Wenzhao Liu}{sch}
\icmlauthor{Haoran Li}{sch}
\icmlauthor{Congying Han}{sch}
\icmlauthor{Zicheng Zhang}{sch}
\icmlauthor{Anqi Li}{sch}
\icmlauthor{Tiande Guo}{sch}
%\icmlauthor{}{sch}
%\icmlauthor{}{sch}
\end{icmlauthorlist}

\icmlaffiliation{sch}{School of Mathematical Sciences, University of Chinese Academy of Sciences, Beijing, China}

\icmlcorrespondingauthor{Congying Han}{ hancy@ucas.ac.cn}

% You may provide any keywords that you
% find helpful for describing your paper; these are used to populate
% the "keywords" metadata in the PDF but will not be shown in the document
\icmlkeywords{Machine Learning, ICML}

\vskip 0.3in
]

% this must go after the closing bracket ] following \twocolumn[ ...

% This command actually creates the footnote in the first column
% listing the affiliations and the copyright notice.
% The command takes one argument, which is text to display at the start of the footnote.
% The \icmlEqualContribution command is standard text for equal contribution.
% Remove it (just {}) if you do not need this facility.

\printAffiliationsAndNotice{}  % leave blank if no need to mention equal contribution
% \printAffiliationsAndNotice{\icmlEqualContribution} % otherwise use the standard text.

\begin{abstract}
% Effective generalization of neural approaches for the Traveling Salesman Problem~(TSP) remains a significant challenge, particularly across varying scales and distributions.
% A major obstacle is that neural networks frequently struggle to learn robust principles for identifying universal structural patterns and deriving optimal solutions from diverse instances.
% A major obstacle is that neural networks often struggle to identify universal structural patterns in TSP, limiting their ability to learn robust principles for tracing optimal solutions in diverse instances.
Achieving generalization in neural approaches across different scales and distributions remains a significant challenge for the Traveling Salesman Problem~(TSP).
A key obstacle is that neural networks often fail to learn robust principles for identifying universal patterns and deriving optimal solutions from diverse instances.
In this paper, we first uncover Purity Law (PuLa), a fundamental structural principle for optimal TSP solutions, defining that edge prevalence grows exponentially with the sparsity of surrounding vertices. 
Statistically validated across diverse instances, PuLa reveals a consistent bias toward local sparsity in global optima.
% This principle prevalently and stably characterizes optimal solutions across different scales and distributions, offering a promising avenue for generalization.
Building on this insight, we propose Purity Policy Optimization~(PUPO), a novel training paradigm that explicitly aligns characteristics of neural solutions with PuLa during the solution construction process to enhance generalization.
Extensive experiments demonstrate that PUPO can be seamlessly integrated with popular neural solvers, significantly enhancing their generalization performance without incurring additional computational overhead during inference.

% 可以直接把The purity orders of edges and their proportion in the optimal solution follow negative exponential functional relationships放进去，但这个需要再升华一下。文章的故事也是如此，我整体看下来感觉逻辑还是比较清晰的，只是缺少一个“直觉上能明白的、合理的、有启发性的”出发点，这个点也就是我们需要的purity law的大白话版本。
% Pula states that the proportion of different edges in the optimal solution exponentially decreases corresponding to the vertex density around those edges.
% In this paper, we first uncover a fundamental structural principle of TSP solutions, termed the Purity Law (PuLa), which delineates that most edges in the optimal solution lack better local surrogates within their immediate neighborhood.
% low vertex density in the local neighborhood of edges prevails in the optimal solutions

% A key obstacle is the inability of neural networks to identify universal structural patterns across diverse instances, notoriously hindering them from learning robust principles for tracing optimal solutions. 
% In this paper, we first introduce purity order and quantitatively uncover Local Purity~(LOP) Law, which states that \textit{the purity orders of edges and their proportion in the optimal solution follow negative exponential functional relationships.}
% We also validate LOP is a fundamental principle that prevalently and stably

\end{abstract}

\section{Introduction}

The Traveling Salesman Problem~(TSP) is a fundamental combinatorial optimization~(CO) problem with broad applications in operations research and computer science fields such as scheduling~\cite{picard1978time,bagchi2006review}, circuit compilation~\cite{Paler_2021,duman2004precedence}, and computational biology~\cite{796584,gusfield2019integer}.
Due to its NP-hard nature, even the most advanced exact solver~\cite{APPLEGATE200911} could not efficiently find the optimal solution for large-scale instances within a reasonable time frame.
As a result, approximate heuristic algorithms, such as LKH3~\cite{helsgaun2000effective, helsgaun2017extension},
% —the state-of-the-art~(SOTA) method—
have been developed to find near-optimal solutions for large instances with improved efficiency.
However, these approaches still face significant computational overhead due to their iterative search processes for each instance.

\begin{figure}[t]
\begin{center}
\centerline{\includegraphics[scale=0.37]{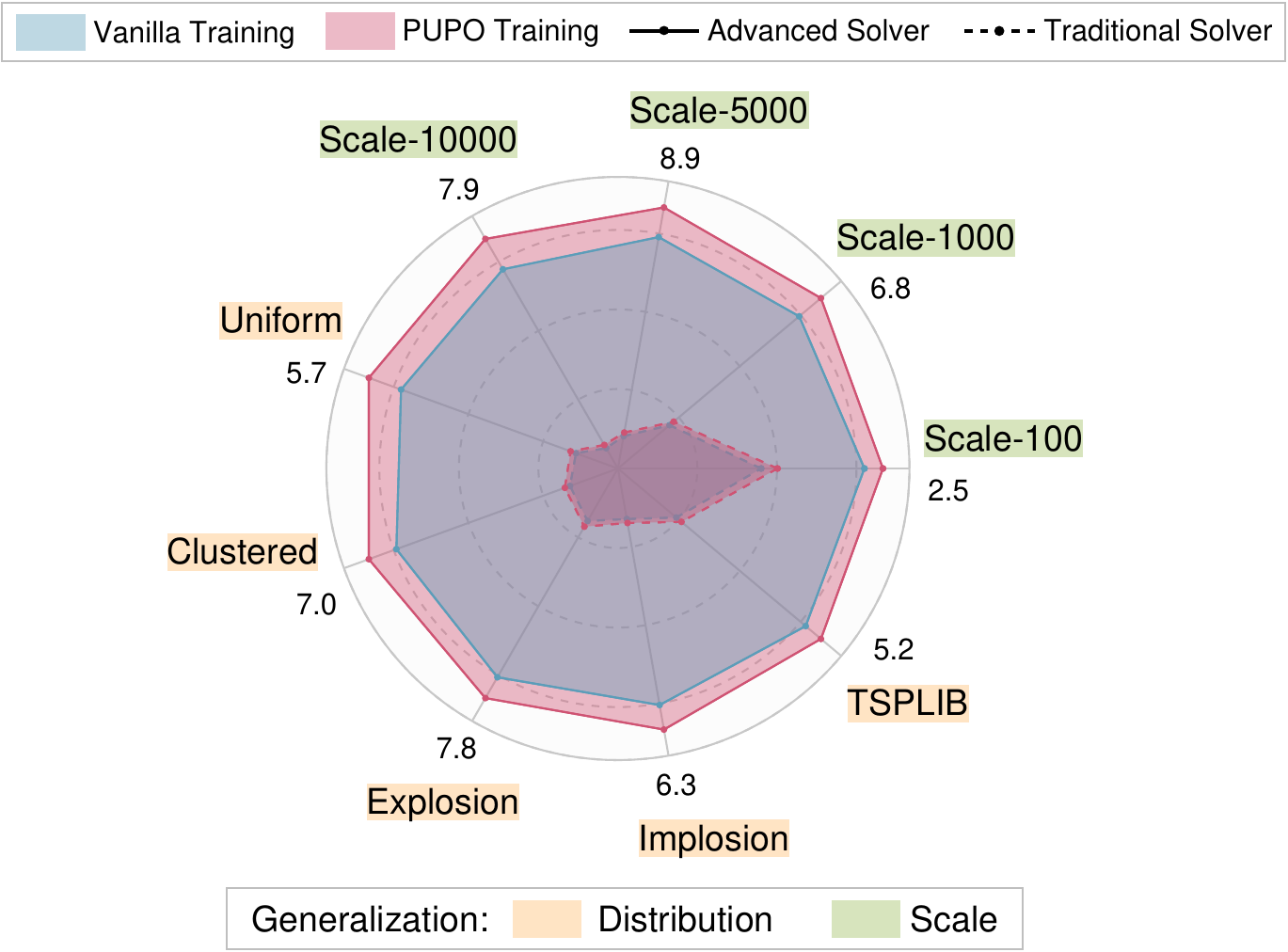}}
% \centerline{\includegraphics[width=\columnwidth]{radar.pdf}}
\vspace{0em}
\caption{Generalization performance of the neural TSP solver with PUPO is significantly improved across various instance scales and distributions compared to the vanilla solver without PUPO.}
\label{radar}
\end{center}
% \vskip -0.2in
\vspace{-3em}
\end{figure}

Deep learning, particularly deep reinforcement learning (DRL), holds significant promise for developing fast and advanced TSP heuristics.
Learning-based approaches generally fall into two categories: methods that learn to construct solutions step by step~\cite{bello2016neural,kool2018attention,kwon2020pomo,jin2023pointerformer}, and methods that learn to search for better solutions iteratively~\cite{lu2019learning, article,fu2021generalize}. Among these, 
neural constructive methods excel in faster inference and higher performance, enabling real-time TSP applications. 
However, they often struggle with generalization and exhibit limited performance on large-scale or heterogeneous instances.

The primary challenge behind this weak generalization lies in the tendency of neural models to overfit to specific patterns tied to particular training settings. 
To address this, some approaches have focused on simplifying the decision space, such as restricting the feasible action set~\cite{pmlr-v235-fang24c} and employing divide-and-conquer strategies~\cite{Pan_Jin_Ding_Feng_Zhao_Song_Bian_2023,NEURIPS2023_1c10d0c0}.
While these methods improve performance, they inherently reduce the potential optimality of the solutions. 
More importantly, they lack guidance from the universal structural properties of TSP, which prevents neural solvers from directly learning consistent and applicable patterns across various instances.

In this paper, we first explore generalizable structural patterns in TSP, then leverage them to guide the learning process of neural solvers. 
To begin with, 
we define edge purity order as a measure of vertex density around edges, where sparser neighborhoods yield lower purity orders.
We reveal Purity Law~(PuLa), a fundamental principle for TSP  stating: \textit{the proportion of different edges in the optimal solution follows a negative exponential law based on their purity orders across various instances.}
This indicates that edges with lower purity orders are more prevalent in optimal solutions, with their frequency increasing exponentially as the order decreases.
PuLa demonstrates its potential to capture universal structural information in two key ways.
First, it is ubiquitously present in optimal solutions in various instances, underscoring \textit{its strong connection to optimality}.
Second, the parameters of its negative exponential model remain statistically invariant across varying instance scales and distributions, proving that \textit{the sparsity-driven edge dominance law is intrinsic rather than data-specific}. 
% These insights are validated through extensive statistical experiments and geometric analysis. 

Based on the universal presence of the Purity Law, we propose \textbf{PU}rity \textbf{P}olicy \textbf{O}ptimization~(PUPO), a novel training framework that incorporates this generalizable structural information into the policy optimization process.
PUPO guides the solution construction process by encouraging the emergence of PuLa. 
Specifically, it modifies the policy gradient to balance the purity metrics of solutions at different decision stages.
By aligning with PuLa, PUPO helps models learn consistent patterns independent of specific instances, thereby improving their ability to generalize to different distributions and scales, without altering the underlying network architecture.
Notably, PUPO can be easily integrated with a wide range of existing neural solvers.
Extensive experiments show that PUPO significantly enhances the generalization ability of these solvers~(shown in Fig.~\ref{radar}), without increasing computational overhead during inference~(shown in Fig.~\ref{time_bar}). 
% The whole improvement is illustrated 
% The radar chart illustrating the overall improvement is shown in Figure \ref{radar}, with the detailed plotting process provided in Appendix \ref{radar_process}.
Our contributions include:
% \vspace{-0.5em}
\begin{itemize}[leftmargin=0em, itemindent=1em,topsep=0em]
\setlength{\itemsep}{1pt}
\setlength{\parsep}{1pt}
\setlength{\parskip}{1pt}
    \item We identify the Purity Law~(PuLa) as a fundamental principle that reliably characterizes optimal solutions across various scales and distributions, offering a novel perspective for comprehending common structural patterns in TSP.
    \item We propose Purity Policy Optimization~(PUPO), a novel training approach that explicitly encourages alignment with PuLa during the solution construction process, helping models to learn consistent patterns across different instances.
    \item We demonstrate through extensive experiments that PUPO can be seamlessly integrated with popular neural solvers to considerably improve their generalization, without increasing computational cost during inference.
\end{itemize}

\section{Related Work}
% Since TSP is a significant problem in combinatorial optimization, there are many studies related to this topic. We mainly review the representative traditional methods, classical learning-based approaches, and recent research to enhance the generalizability of neural models.

\textbf{Traditional Methods.} 
TSP solvers are typically categorized into exact, approximate, and heuristic. 
Concorde~\cite{APPLEGATE200911}, a leading exact solver, formulates TSP as a mixed integer program and uses a branch and cut algorithm~\cite{padberg1991branch} to solve it.
Christofides Algorithm~\cite{Christofides1976WorstCaseAO}, a representative approximate method, constructs solutions starting from a minimum spanning tree and finding a minimum cost perfect matching in a specific induced subgraph, achieving an approximation ratio of 1.5.
LKH-3~\cite{helsgaun2017extension} is the SOTA heuristic method, employing local search and k-opt exchanges to iteratively improve solutions. 
However, achieving high-precision solutions with LKH-3 incurs significant computational costs, making it impractical for large-scale instances.

\textbf{Classical Neural Solvers.} 
With the rapid progress in deep learning, various neural approaches to combinatorial optimization have emerged.
Overviews of these methods can be found in \citet{guo2019solving} and \citet{bengio2021machine}.
These approaches can generally be divided into two classes: learning to directly construct and learning to iteratively improve solutions.
For constructive methods, \citet{vinyals2015pointer} introduces the Pointer Network, which solves TSP end-to-end utilizing Recurrent Neural Networks to encode vertex embeddings, trained via supervised learning.
Transformer-based architectures in DRL further improve the encoding of node features globally and employ autoregressive decoding to generate solutions~\cite{kool2018attention,kwon2020pomo,bresson2021transformer}.
Pointerformer~\cite{jin2023pointerformer} enhances memory efficiency by adopting a reversible residual network in the encoder and a multi-pointer network in the decoder.
For iterative methods, DRL-2opt~\cite{article} trains a DRL policy to select appropriate 2-opt operators for refining solutions.
GCN+MCTS~\cite{fu2021generalize} integrates graph decomposition and Monte Carlo Tree Search to handle large-scale instances effectively.

\textbf{Advanced Solvers.} 
Traditional Neural methods often struggle with poor generalization~\cite{joshi2021learning}.
UTSP~\cite{min2024unsupervised} applies unsupervised learning by leveraging permutation matrix properties but fails to scale to instances larger than 2000 vertices.
H-TSP~\cite{Pan_Jin_Ding_Feng_Zhao_Song_Bian_2023} improves efficiency for large-scale instances by employing hierarchical DRL. 
LEHD~\cite{NEURIPS2023_1c10d0c0} and GLOP~\cite{Ye_Wang_Liang_Cao_Li_Li_2024} continuously refine solutions through divide-and-conquer strategies. 
PDAM~\cite{Zhang_Xiao_Wang_Song_Chen_2023} and  INViT~\cite{pmlr-v235-fang24c} simplify the decision space by restricting actions to local neighborhoods, which improves generalization. 
However, they inherently trade off the optimality of solutions in favor of generalization.
Despite progress, no research has explored universal structural principles of TSP to guide neural learning, leaving the challenge of identifying consistent instance patterns unsolved.

\section{Preliminary}

In this section, we introduce the formal definition of TSP and the policy learning paradigm for TSP solutions.

\subsection{Traveling Salesman Problem}
In this paper, we focus on the Euclidean TSP in two-dimensional space. 
Given a TSP instance, let $G = (\mathcal{X}, \mathcal{E}, N)$ represents an undirected graph, where $N$ is the number of vertices, $\mathcal{X} = \{x_i|1\leq i\leq N \}$ is the vertex set, and $\mathcal{E} = \{e_{ij} = (x_i, x_j)|1\leq i, j\leq N\}$
is the edge set. 
In the Euclidean TSP, $G$ is fully connected and symmetric, with each vertex $x_i \in \mathcal{X}$ represented by a coordinate $(x_i^1, x_i^2)$ scaled to the unit square $[0, 1]$. 
Each edge $e_{ij} \in \mathcal{E}$ is assigned a traversal cost $c(x_i, x_j)$, typically the Euclidean distance between the vertices $x_i$ and $x_j$. 
A Hamiltonian cycle $\boldsymbol{x} = ( x_{\tau_1}, x_{\tau_2}, \cdots, x_{\tau_N} )$ is a tour that visits each vertex exactly once. The total cost of $\omega$ is formulated as follows:
\begin{equation}
L(\boldsymbol{x}) =  c(x_{\tau_N}, x_{\tau_1}) + \sum_{i=2}^{N}c(x_{\tau_{i-1}}, x_{\tau_{i}}), 
\end{equation}
The objective of solving TSP is to find the Hamiltonian cycle with the minimum total cost among all feasible cycles.

\subsection{Policy Learning for the TSP}

Policy learning is the dominant paradigm in DRL for solving the TSP. Given an instance $\mathcal{X}$, a solution $\boldsymbol{\tau} = (\tau_1, \dots, \tau_N)$ is autoregressively generated through the neural policy:
\begin{equation}
    p_{\theta}(\boldsymbol{\tau} | \mathcal{X}) = p_{\theta}(\tau_1 | \mathcal{X}) \prod \limits_{t=2}^N p_{\theta}(\tau_{t} | \tau_{1:t-1}, \mathcal{X}), 
\end{equation}
where $\theta$ are the parameters of the policy network, trained by minimizing the expected tour length: 
\begin{equation}
    \mathcal{L} (\theta | \mathcal{X}) = \mathbb{E}_{p_{\theta}(\boldsymbol{\tau} | \mathcal{X})} [L(\boldsymbol{\tau}) ]. 
\end{equation}
REINFORCE~\cite{williams1992simple} is the cornerstone of policy-based reinforcement learning methods, which computes the gradient $\nabla \mathcal{L} (\theta | \mathcal{X})$ in the context of TSP as the following:
\begin{equation} \label{eq: reinforce tsp}
\scalebox{0.95}{$
    \mathbb{E}_{p_{\theta}(\boldsymbol{\tau} | \mathcal{X})} \left[\left( L(\boldsymbol{\tau}) - b(\mathcal{X}) \right) \sum \limits_{t=2}^{N} \nabla  \log p_{\theta}(\tau_{t} | \tau_{1:t-1}, \mathcal{X})  \right],
    $}
\end{equation}
where $b(\mathcal{X})$ is a baseline for variance reduction.
    
% \begin{equation}
%     \begin{aligned}
%     &\nabla \mathcal{L} (\theta | \mathcal{X}) \\
%     =& \mathbb{E}_{p_{\theta}(\boldsymbol{\tau} | \mathcal{X})} \left[L(\boldsymbol{\tau})  \nabla \log p_{\theta}(\boldsymbol{\tau} | \mathcal{X}) \right] \\
%     =& \mathbb{E}_{p_{\theta}(\boldsymbol{\tau} | \mathcal{X})} \left[L(\boldsymbol{\tau}) \sum \limits_{t=1}^{n-1} \nabla  \log p_{\theta}(\tau_{t} | \tau_{1:t-1}, \mathcal{X})  \right].  \\ \nonumber
%     \end{aligned}
% \end{equation}

% To reduce the training variance, REINFORCE with a baseline is more commonly adopted, with $L(\boldsymbol{\tau})$ replaced by $L(\boldsymbol{\tau}) - b(\mathcal{X})$.

\section{Purity Patterns for Generalization}
In this section, we first explore the structural patterns closely tied to generalization in optimal solutions. 
We observe that edges surrounded by sparse vertices consistently prevail across different instances. 
To capture this relationship, we propose the concept of \textit{purity order}, which quantifies the density of vertices around an edge.
Based on this, we uncover the Purity Law, an empirical phenomenon that reveals a universal structural principle in TSP optimal solutions.

\subsection{Purity Order in Traveling Salesman Problems}

We introduce the concept of purity order to formulate the vertex density around each edge.

\begin{definition}[Purity Order]
The \textit{covering set} $N_{c}$ and \textit{purity order} $K_p$ of an edge $e_{ij} = (x_i, x_j)$ are defined as:
\begin{subequations}
\begin{align}
    N_{c}(e_{ij}) &:= \{x \in \mathcal{X} \mid (x_i-x)^T \cdot (x_j-x)<0 \}, \label{eq:a} \\
    K_p(e_{ij}) &= K_p(x_i, x_j) := |N_{c}(e_{ij})|. \label{eq:b}        
\end{align}
\end{subequations}
\end{definition}
Geometrically, the covering set of an edge corresponds to the vertices lying within a circle, whose diameter is defined by this edge, and the purity order is the cardinality of this set. 
As illustrated in Fig.~\ref{same length diff order}, the purity order reflects the local density of vertices surrounding an edge.
A lower purity order indicates a more sparse and “pure” configuration of vertices around the edge. 
For convenience, we refer to an edge with purity order $k$ as a $k$-order pure edge. 

Additionally, we investigate the topological properties of $0$-order pure edges with the lowest redundancy (details shown in Appendix~\ref{proof topo}).
Specifically, we establish the \textit{existence} of $0$-order pure neighbors for any vertex, laying the groundwork for our optimization paradigm.
We then prove the \textit{connectivity} of the subgraph formed by $0$-order pure edges, indicating that these edges capture global structural properties.
Furthermore, we demonstrate that the polyhedron formed by the $0$-order pure neighbors of any vertex is \textit{convex}, highlighting its structural integrity and stability. This insight suggests the potential for efficient computational methods. 
% In particular, a $0$-order pure edge is called a strongly pure edge.

\begin{figure}[t]
% \vskip 0.2in
\begin{center}
\centerline{\includegraphics[scale=0.4]{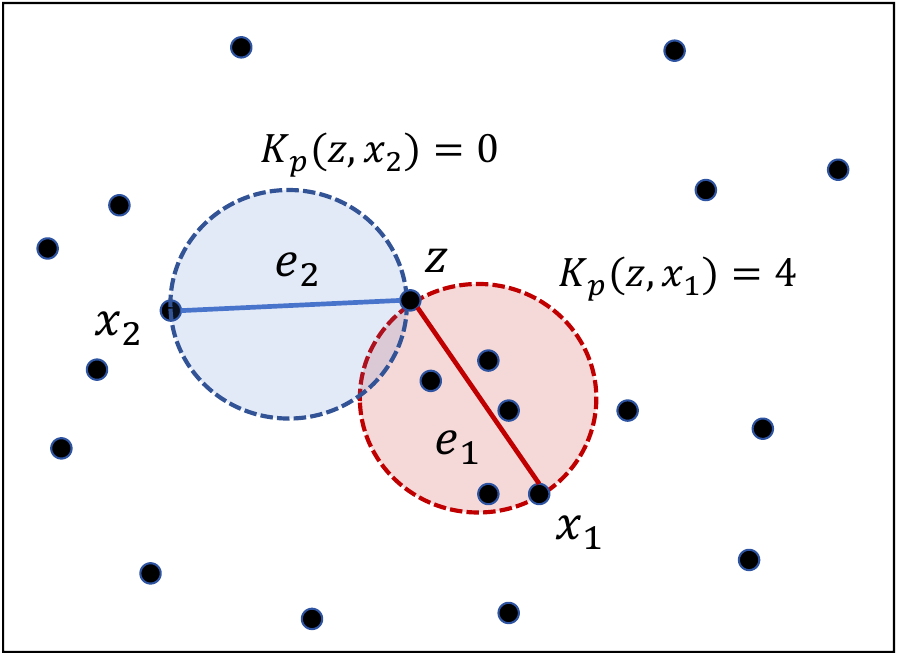}}
\caption{Example of edge purity orders with the same length.}
\label{same length diff order}
\end{center}
\vskip -0.2in
\end{figure}

\textbf{Intuitive Example Analysis.}
Purity order not only measures vertex redundancy around an edge, but also captures topological information beyond edge length.
Even among edges of equal length, those with lower purity orders typically exhibit more favorable structural properties.
As shown in Fig.~\ref{same length diff order}, consider two edges, $e_1$ and $e_2$, formed with vertices $x_1$ and $x_2$ from vertex $z$.
Despite having the same length, their purity orders are 4 and 0, respectively.
The edge $e_1$ passes through a denser region of vertices, which negatively affects the subsequent connections, making it harder to form a cohesive structure. 
In contrast, $e_2$, surrounded by sparser vertices, has minimal negative impact on the underlying connectivity of other vertices, suggesting more conducive to a well-structured solution.
These observations suggest that solutions with more lower-order pure edges may have a greater potential for optimality.

% In the following, we will illustrate the insight behind low-order pure edges and quantify the relationship between purity order and the optimal solution from a statistical perspective.

\subsection{Purity Law in TSP Optimal Solutions}
\label{subsection: purity law}
We quantitatively investigate the distribution of purity orders in optimal solutions across various instance scales and distributions, leading to the concept of the Purity Law.

%{colbacktitle=red!10!white,colback=blue!4!white,coltitle=red!70!black,title={#2},fonttitle=\bfseries,#1}
\definecolor{macaronblue}{rgb}{0.502, 0.702, 1}
\definecolor{darkblue}{rgb}{0, 0, 0.55} 
\newtcolorbox{mybox}[2][]{colbacktitle=red!10!white,
colback=macaronblue!10!white,coltitle=darkblue!80!white,
title={#2},fonttitle=\bfseries,#1}
\begin{mybox}[attach title to upper={\  :\ }]{Purity Law}
The distribution of edges in optimal solutions follows a \textit{negative exponential law} based on their purity orders across various instances.
\end{mybox}

To validate the Purity Law, we conduct extensive statistical experiments on instances with varying scales and distributions. 
Specifically, we create datasets from 84 instance types with scales ranging from $20$ to $1000$, sampled from four classical distributions. 
The corresponding optimal solutions are computed using the LKH3 algorithm~\cite{helsgaun2017extension}.
A detailed description of the dataset is provided in Append.~\ref{stat dataset}.

\begin{table}[t]
    \vskip -0.15in
\caption{Fitting results of the  exponential function in Eq. \eqref{eq: pula}.}\label{tab_error}
    \begin{center}
    \begin{small}
    \begin{sc}
    \begin{tabular}{lccc}
    \toprule
        ~ & Fitting Error & $\alpha$ & $\beta$  \\ \midrule
        Mean & 2.23E-05 & 0.92 & 2.63  \\ 
        Variance & 8.45E-10 & 1.57E-04 & 1.49E-02  \\ \bottomrule
    \end{tabular}
    \end{sc}
    \end{small}
    \end{center}
\end{table}

\begin{figure}[t]
% \vskip 0.2in
\begin{center}
\centerline{\includegraphics[width=\columnwidth]{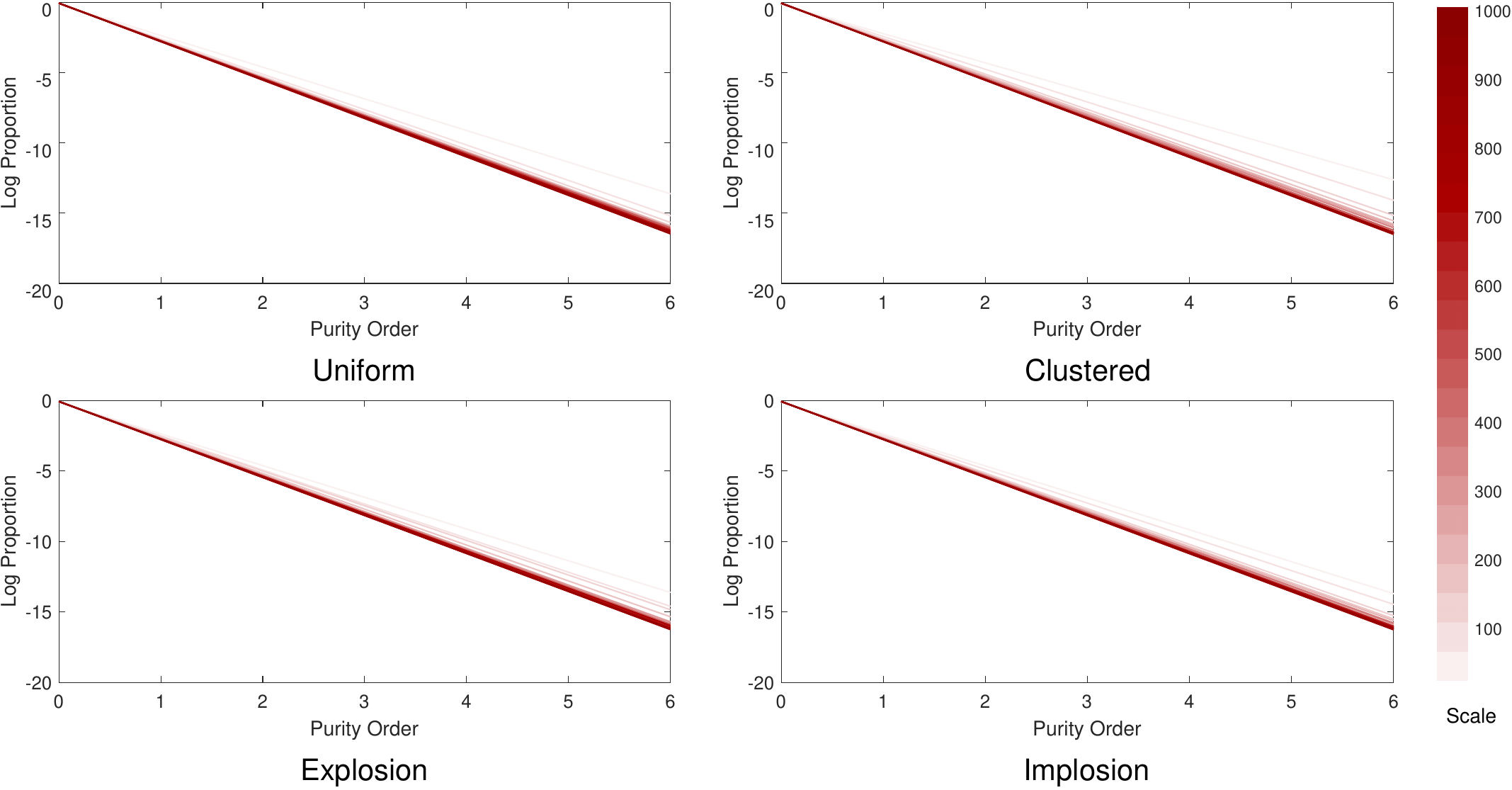}}
\vskip -0.1in
\caption{Purity Law curves under varying scales and distributions.}
\label{ratio law}
\end{center}
\vskip -0.2in
\end{figure}

For each instance type, we calculate the purity order of every edge in the optimal solution and then statistic the proportion $y$ of edges with a given purity order $k$, which is defined as:
\begin{equation}
\scalebox{0.95}{$
    y(k)=\frac{\sum_{(\tau_i^*, \tau_{j}^*) \in \boldsymbol{\tau^*}}\mathbb{I}(K_p(\tau_i^*, \tau_{j}^*) = k)}{N},\quad \ k \in [0, N-2],
    $}
\end{equation}
where $N$ is the instance scale, $(\tau_i^*, \tau_{j}^*)$ is the edge in the optimal solution $\boldsymbol{\tau^*}$, and $\mathbb{I}(\cdot)$ is the indicator function.
Next, we fit the proportions to a negative exponential function:
\begin{equation}\label{eq: pula}
    \log y(k)=-\beta k + \log\alpha,
\end{equation}
and present mean and variance of fitted parameters $\alpha$ and $\beta$, along with fitting errors, in Table \ref{tab_error}. 
Further fitting results for each instance type are available in Appendix~\ref{fitting result}.
The remarkably low mean and variance of fitting errors underscore that the negative exponential law reliably and universally applies across different instance scales and distributions.

The negative exponential relationship reveals that purity orders of edges and their proportions in the optimal solution decrease exponentially, confirming our earlier comprehension.
To better understand the decay rate, we visualize the fitted curves for different instance types in Fig.~\ref{ratio law}, where we plot the logarithm of $y$ for clarity.
As seen in Fig.~\ref{ratio law}, the coefficient $\alpha$ is invariant for both scales and distributions, with different curves nearly overlapping at the vertical axis intersection. 
The decay rate $\beta$ slightly increases with instance scale, indicating that, for larger-scale instances, the proportion of each purity order decreases more rapidly.

\subsection{Consistent Dominance of Lowest-order Pure Edges}

\begin{figure}[t]
\begin{center}
% \centerline{\includegraphics[width=\columnwidth]{0order_ratio_new.pdf}}
\includegraphics[width=\linewidth]{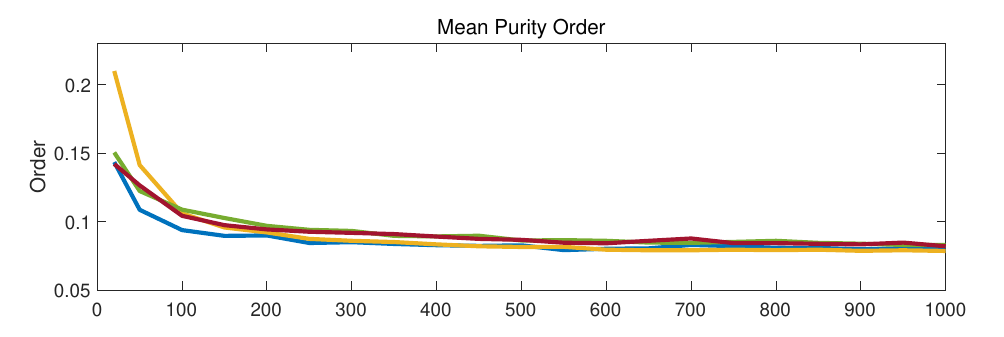}
\includegraphics[width=\linewidth]{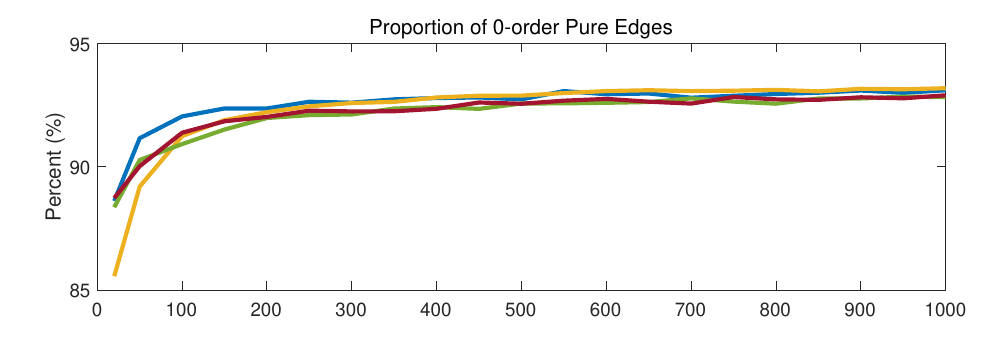}
\includegraphics[width=\linewidth]{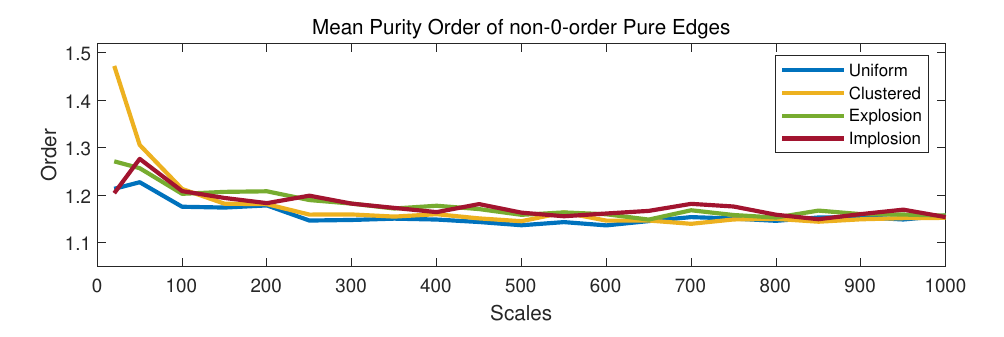}
\vskip -0.2in
\caption{Mean purity order, proportion of $0$-order pure edges, and average order of non-$0$-order pure edges for varying instance.}
\label{0order_ratio}
\end{center}
\vskip -0.2in
\end{figure}

% \begin{figure*}[t]
%     \centering
%     \subfigure[Mean Purity Order]{
%         \includegraphics[width=0.3\textwidth]{sub11.pdf}
%         \label{fig:image1}
%     }
%     \subfigure[Proportion of $0$-order Pure Edges]{
%         \includegraphics[width=0.31\textwidth]{sub21.pdf}
%         \label{fig:image2}
%     }
%     \subfigure[Mean Order of non-$0$-order Pure Edges]{
%         \includegraphics[width=0.31\textwidth]{sub31.pdf}
%         \label{fig:image3}
%     }
%     \caption{Mean purity order, proportion of $0$-order pure edges, and average order of non-$0$-order pure edges for each instance type.}
%     \label{fig:all_images}
% \end{figure*}

The rapid decay of the negative exponential function ensures that edges with lower purity orders dominate the optimal solution, while edges with higher purity orders are seldom present.
To concretely illustrate this property, we calculate the mean purity order, the proportion of $0$-order pure edges, and the average order of non-$0$-order pure edges in optimal solutions for each instance type. The results, shown in Fig.~\ref{0order_ratio}, highlight the consistent dominance of $0$-order pure edges.

The variations in these metrics exhibit a steady trend across different instance scales and distributions, further reinforcing the universality and stability of the Purity Law. 
The mean purity order stabilizes around $0.1$, indicating that the optimal tours consistently maintain a quite low level of overall purity order. 
Furthermore, the proportion of $0$-order pure edges consistently stays around $92 \%$, reflecting their dominant presence in optimal solutions. This proportion also aligns with the fitting parameter $\alpha$. 
For non-$0$-order pure edges, which have a minimal purity order of $1$, the mean purity order remains around $1.15$. 
This suggests that, even among non-$0$-order pure edges, low-order pure edges are predominant, with higher-order pure edges being quite rare.
Notably, all these proportions remain nearly invariant when the instance scale exceeds $300$, underscoring the remarkable and dominant consistency of low-order pure edges.

% In the gradient of the policy network, the trajectory reward is the exclusive source of information related to the TSP solution, which implies that during model training, the tour length is the only guideline driving the model's learning process.
% However, it only characterizes the optimization objective while failing to represent the cross-instance pattern properties inherent in the TSP.
% 在策略网络的梯度中，涉及到TSP解信息的只有轨迹的奖励。这意味着在模型训练过程中，环游长度是用来指导模型学习的唯一信息。但它只能刻画优化目标，并不能够表征TSP中跨实例的模式信息。

%尽管之前的研究工作在修改策略网络的架构与学习模块方面下了很大功夫，但由于缺乏
%使得模型仍不能很好地学到不依赖于训练实例的通用模式
\section{Policy Learning Inspired by Purity Patterns}

In standard neural learning for TSP, the tour length, as defined in Eq.~\eqref{eq: reinforce tsp}, is typically the only objective for policy optimization. 
However, relying solely on this reward signal, neural networks often fail to capture inherent structural patterns present across different instances. 
Despite efforts to modify the architecture of policy networks and learning modules in previous research, the absence of guidance from generalizable structural principles, limits effective learning of universal patterns, resulting in poor generalization.

Fortunately, the widespread and stable presence of purity patterns in optimal solutions presents a promising avenue for improving generalization.
Motivated by this, we introduce two key concepts--\textit{purity availability} and \textit{purity cost}--to characterize purity patterns for neural learning. 
By incorporating these into the policy training framework, we propose purity policy optimization to encourage alignment with the Purity Law, enhancing generalization in neural solvers.

\subsection{Purity Availability and Purity Cost}

In the Markov model of TSP, the state at time $t$ consists of the \textit{unvisited vertex set} $\mathcal{U}_t$ and the \textit{current partial solution} $\boldsymbol{\tau}_t$, which includes the visited vertex set $\mathcal{V}_t$.
The action $a_t$ corresponds to the vertex $\tau_t$ selected for the next visit. 
\begin{definition}[Purity Availability]
    The \textit{purity availability} $\phi(\cdot)$ of the unvisited vertex set $\mathcal{U}_t$ is defined as the average minimum available purity order, given by:
\begin{equation}
\scalebox{0.97}{$
    \phi (\mathcal{U}_t) = \frac{\sum_{x_i \in \mathcal{U}_t} \mathop{\min}_{\substack{x_j \in \mathcal{U}_t, j \neq i}} K_{p}(x_i, x_j)}{|\mathcal{U}_{t}|}, \quad 1\le t \le N. 
    $}
\end{equation}
\end{definition}
The purity availability $\phi (\mathcal{U}_t)$ measures the potential purity order of future edges that can be formed by the unvisited vertices.
We demonstrate that the $\phi$ is supermodular, meaning the absolute value of marginal gain in purity availability diminishes as the set $\mathcal{U}_t$ increases. The proof is in Appendix~\ref{proof_subm}.

\begin{proposition}[Supermodularity of Purity Availability]\label{prop: subm}
    The set function $\phi: 2^\mathcal{X} \to \mathbb{R} $, defined on subsets of the finite set $\mathcal{X}$, is supermodular. Specifically, for any subsets $\mathcal{A} \subseteq \mathcal{B} \subseteq \mathcal{X}$ and any vertex $x \in \mathcal{X} \setminus \mathcal{B}$, the following holds:
    \begin{equation}
        \scalebox{1.00}{$
            \phi(\mathcal{A} \cup \{x\}) - \phi(\mathcal{A}) \leq \phi(\mathcal{B} \cup \{x\}) - \phi(\mathcal{B}).
            $}
    \end{equation}
\end{proposition}
This indicates that, as the solution construction progresses, the absolute value of marginal improvement in purity availability becomes more significant on smaller unvisited sets.

\begin{definition}[Purity Cost]
    The \textit{purity cost} $ C(\mathcal{U}_t, \tau_{t+1})$ for selecting an action $\tau_{t+1}$ at time $t$ is defined as:
    % For any $1 \le t \le N$, let $\mathcal{U}_{t+1} = \mathcal{U}_t \setminus {\tau_{t+1} }$. The \textit{purity cost} $ C(\mathcal{U}_t, \tau_{t+1})$ is defined: 
    \begin{equation} 
        \begin{cases}
          K_p(\tau_t, \tau_{t+1}) + {\phi (\mathcal{U}_{t+1})} - {\phi (\mathcal{U}_{t})}, & \text{for } t < N, \\
          K_p(\tau_N, \tau_{1}), & \text{for } t = N. 
        \end{cases}
    \end{equation}
\end{definition}
The purity cost $C(\mathcal{U}_t, \tau_{t+1})$ includes both the purity order of the new edge formed by $\tau_t$ and $\tau_{t+1}$, and the difference in purity availability before and after the action $\tau_{t+1}$.
This formulation captures both the contribution of action $\tau_{t+1}$ to the purity order of the current partial solution and its impact on the purity potential of the remaining unvisited vertices.

\subsection{Policy Optimization with Purity Weightings}

To improve generalization, we integrate purity costs into the policy optimization process, aligning the model with the Purity Law. This encourages the learning of structural patterns that are consistent across varying TSP instances.
To assess the potential of the current state-action pair in fostering a low-purity structure for future solution construction, we introduce the purity weighting based on purity costs.
\begin{definition}[Purity Weightings]
    Let $\delta$ be a discount factor.
    % and $t$ be the current time step. 
    The \textit{purity weighting} $ W(\mathcal{U}_t, \tau_{t+1}) $ is defined as: 
    \begin{equation} 
    W (\mathcal{U}_t, \tau_{t+1}) = 1 + \sum \limits_{j=t}^{N} \delta^{j-t} C(\mathcal{U}_j, \tau_{j+1}).  
    \end{equation} 
\end{definition}

Building on this characterization, we propose the \textit{purity policy gradient} $\nabla \mathcal{L}_{PUPO} (\theta | \mathcal{X})$ as the following formulation: 
\begin{equation} 
    \mathbb{E}_{p_{\theta}(\boldsymbol{\tau} | \mathcal{X})} \Biggl[ L(\boldsymbol{\tau})  \sum \limits_{t=2}^{N} W_{t} \nabla \log p_{\theta}(\tau_{t} | \tau_{1:t-1}, \mathcal{X})  \Biggr],
\end{equation}
where $W_{t+1}$ denotes the shorthand for $W(\mathcal{U}_t, \tau_{t+1})$.

\begin{table*}[t]
% \vskip -0.1in
\caption{The experimental results of average gap ($\%$) on the randomly generated dataset with different distributions and scales after training each model using both the vanilla and PUPO methods, where - means out of memory, bold formatting represents superior results. It can be observed that PUPO training enhances the performance of six models across nearly all instance type.}\label{result_rand}
\vskip 0.05in
\setlength{\tabcolsep}{1mm}{
\begin{tabular*}{\linewidth}{@{}cc||cc|cc||cc|cc||cc|cc}
\toprule
\multirow{2}{*}{Distribution} & \multirow{2}{*}{Scale} & \multicolumn{2}{c|}{AM-50}                & \multicolumn{2}{c||}{AM-100}         & \multicolumn{2}{c|}{PF-50}                & \multicolumn{2}{c||}{PF-100}         & \multicolumn{2}{c|}{INViT-50}             & \multicolumn{2}{c}{INViT-100} \\ \cline{3-14} 
 &  & Vanilla       & \multicolumn{1}{c|}{PUPO} & Vanilla & \multicolumn{1}{c||}{PUPO} & Vanilla       & \multicolumn{1}{c|}{PUPO} & Vanilla & \multicolumn{1}{c||}{PUPO} & Vanilla       & \multicolumn{1}{c|}{PUPO} & Vanilla    & PUPO             \\ \cline{3-14}
\multirow{4}{*}{Uniform}      & 100                    & \textbf{5.21} & 5.35                      & 5.07    & \textbf{4.61}             & \textbf{4.16} & 4.31                      & 3.64    & \textbf{3.39}             & \textbf{2.14} & 2.17                      & 2.61       & \textbf{2.35}    \\
& 1000                   & 35.51         & \textbf{34.01}            & 31.23   & \textbf{30.16}            & 41.00         & \textbf{39.10}            & 30.57   & \textbf{29.93}            & 7.27          & \textbf{7.17}             & 7.41       & \textbf{6.26}    \\
& 5000                   & 67.99         & \textbf{65.12}            & 80.64   & \textbf{68.42}            & 113.55        & \textbf{104.09}           & 80.59   & \textbf{77.53}            & 9.41          & \textbf{9.09}             & 9.39       & \textbf{8.05}    \\
& 10000                  & -             & -                         & -       & -                         & -             & -                         & -       & -                         & 7.89          & \textbf{7.63}             & 7.51       & \textbf{6.25}    \\ \cline{3-14}
\multirow{4}{*}{Clustered}    & 100                    & \textbf{7.75} & 8.03                      & 7.66    & \textbf{6.07}             & \textbf{7.67} & 7.71                      & 7.24    & \textbf{6.90}             & 2.90          & \textbf{2.86}             & 3.59       & \textbf{3.22}    \\
& 1000                   & 41.14         & \textbf{39.82}            & 37.99   & \textbf{33.84}            & 51.04         & \textbf{49.58}            & 47.50   & \textbf{44.11}            & 8.44          & \textbf{8.00}             & 8.37       & \textbf{7.20}    \\
& 5000                   & 77.78         & \textbf{75.93}            & 76.37   & \textbf{70.03}            & 137.51        & \textbf{131.61}           & 137.15  & \textbf{121.79}           & 9.90          & \textbf{9.45}             & 9.65       & \textbf{8.55}    \\
& 10000                  & -             & -                         & -       & -                         & -             & -                         & -       & -                         & 10.85         & \textbf{10.09}            & 10.47      & \textbf{9.11}    \\ \cline{3-14}
\multirow{4}{*}{Explosion}    & 100                    & \textbf{5.01} & 5.18                      & 5.04    & \textbf{4.71}             & \textbf{4.98} & 5.12                      & 4.74    & \textbf{4.65}             & 2.15          & \textbf{2.11}             & 2.70       & \textbf{2.41}    \\
& 1000                   & 35.57         & \textbf{35.07}            & 32.97   & \textbf{30.20}            & 46.04         & \textbf{45.34}            & 41.60   & \textbf{37.73}            & 9.61          & \textbf{9.18}             & 9.61       & \textbf{8.77}    \\
& 5000                   & 71.95         & \textbf{66.66}            & 74.55   & \textbf{66.33}            & 133.99        & \textbf{129.02}           & 120.96  & \textbf{110.16}           & 12.44         & \textbf{11.65}            & 11.45      & \textbf{10.41}   \\
& 10000                  & -             & -                         & -       & -                         & -             & -                         & -       & -                         & 11.22         & \textbf{10.61}            & 10.96      & \textbf{9.62}    \\ \cline{3-14}
\multirow{4}{*}{Implosion}    & 100                    & \textbf{5.02} & 5.17                      & 5.12    & \textbf{4.64}             & \textbf{4.71} & 4.91                      & 4.49    & \textbf{4.35}             & 2.44          & \textbf{2.38}             & 2.77       & \textbf{2.62}    \\
& 1000                   & 37.38         & \textbf{37.12}            & 32.44   & \textbf{30.94}            & 46.54         & \textbf{44.83}            & 36.11   & \textbf{34.47}            & 7.56          & \textbf{7.50}             & 7.81       & \textbf{6.90}    \\
& 5000                   & 78.57         & \textbf{74.05}            & 72.87   & \textbf{66.58}            & 151.37        & \textbf{138.71}           & 120.67  & \textbf{103.74}           & 9.20          & \textbf{8.86}             & 9.44       & \textbf{8.96}    \\
& 10000                  & -             & -                         & -       & -                         & -             & -                         & -       & -                         & 8.36          & \textbf{7.68}             & 8.27       & \textbf{6.78} \\
     \bottomrule
\end{tabular*}%
}
% \vskip -0.1in
\end{table*}

The complete description of the Purity Policy Optimization~(PUPO) framework is presented in Appendix~\ref{PUPO_alg}. 
PUPO establishes upon the REINFORCE algorithm with a baseline, incorporating generalizable purity patterns into the modified policy gradient.
At each stage of the solution construction, PUPO employs a discounted purity cost that reflects the purity potential of both the current partial solution and the unvisited vertex set, encouraging the model to favor actions with lower purity orders during training. 
This approach helps the model to learn consistent, cross-instance patterns that align with the purity law, thereby enhancing its generalization capability.
Notably, PUPO is flexible and can be integrated with various popular constructive neural solvers, without any alterations to the network architecture.

\section{Numerical Experiments}

%To validate the impact of PUPO on enhancing generalizability, we conduct both original policy optimization and pure policy optimization on several SOTA constructive models. We also performed a comprehensive evaluation of generalization performance and purity metrics on well-known public datasets.

To validate the effect of PUPO on enhancing generalization, we compare the performance of both vanilla policy optimization and purity policy optimization across several state-of-the-art constructive neural solvers. Additionally, we conduct a comprehensive evaluation of generalization performance and purity metrics on well-known public datasets.

\subsection{Experimental Setups}
\textbf{Dataset.} 
To ensure the adequacy of the experiments, we validate the model's performance on two categories of datasets. The randomly generated dataset used in this paper is the same as that in INViT \cite{pmlr-v235-fang24c}, which is widely adopted to testify existing DRL approach.
It contains 16 subsets and corresponding (near-)optimal solutions for TSP, including 4 distributions (uniform, clustered, explosion, and implosion) and 4 scales (100, 1000, 5000 and 10000).
The real-world dataset we used is TSPLIB, a well-known TSP library~\cite{reinelt1991tsplib} that contains 100 instances with various nodes distributions and their optimal solutions. These instances come from practical applications with scale ranging from 14 to 85,900. In our experiment, we consider all instances with no more than 10000 nodes.

\textbf{Comparison Methods.}
Although PUPO can be integrated with any DRL-based neural constructive model, we select three representative SOTA methods with high recognition for our experiments: 
(1) AM~\cite{kool2018attention}, a high-performance neural constructive approach widely regarded as a benchmark, which first incorporates the transformer to solve the TSP; 
(2) PF~\cite{jin2023pointerformer}, a neural constructive approach that designs reversible residual networks and multi-pointer networks to efficiently limit memory consumption; 
%(ICML 2024)
(3) INViT~\cite{pmlr-v235-fang24c}, the latest method that enhances generalization ability by constraining the action space and employing multi-view encoding. 
For each method, we conduct training on scales of 50 and 100, resulting in a total of 6 comparison configurations: AM-50, AM-100, PF-50, PF-100, INViT-50, and INViT-100.

\textbf{Evaluation Metrics.} 
We report three widely adopted performance metrics to evaluate the generalizability of each comparison method, including the average length of solutions, the average gap to the optimal solutions, and the average solving time.
The gap characterizes the relative difference between the output solution of neural models and the optimal solution, which is calculated as the following:
\begin{equation}
gap=\frac{L(\boldsymbol{\tau}^{model})-L(\boldsymbol{\tau}^{opt})}{L(\boldsymbol{\tau}^{opt})} \times 100 \% 
\end{equation}
We also present three metrics to reflect the ability of the model to perceive purity information: the average purity order (APO (all) for short), the proportion of 0-order pure edges (Prop-0 (\%) for short), and the average order of non-0-order pure edges (APO (non-0) for short) for each solution.

\textbf{Experimental Settings.}
All the numerical experiments are implemented on an NVIDIA GeForce RTX 3090 GPU with 24 GB of memory, paired with a 12th Gen Intel(R) Core(TM) i9-12900 CPU.
%All the numerical experiments are implemented on an NVIDIA GeForce RTX 3090 GPU with 24 GB and a 12th Gen Intel(R) Core(TM) i9-12900 CPU.
We train each model using the vanilla REINFORCE algorithm and PUPO, respectively.
In each training paradigm, we use the original network architectures provided in the source code, \textit{without any modifications to modules}.
%In each training paradigm, we all use network provided by the source code of each model, \textit{without any modification of modules}.
Additionally, all training procedures and hyper-parameter remain completely consistent with the original source code, except for the learning rate.
Due to the difference in the policy gradient between PUPO and the original method, we adjust the learning rates for the six models during PUPO. 
%Due to the value of the policy gradient of PUPO unlike the original gradient
Specifically, the learning rates are set to 0.0001, 0.00015, 0.0001, 0.00017, 0.00011, 0.00012 for AM-50, AM-100, Pointerformer-50, Pointerformer-100, INViT-50, and INViT-100, respectively.
%Due to the value of the policy gradient of PUPO unlike the original gradient, we set the learning rates for the six models during PUPO to 0.0001, 0.00015, 0.0001, 0.00017, 0.00011, 0.00012, for AM-50, AM-100, Pointerformer-50, Pointerformer-100, INViT-50, and INViT-100.
All models are trained on instances randomly sampled from a uniform distribution.

\begin{figure}[t]
\begin{center}
\centerline{\includegraphics[width=\columnwidth]{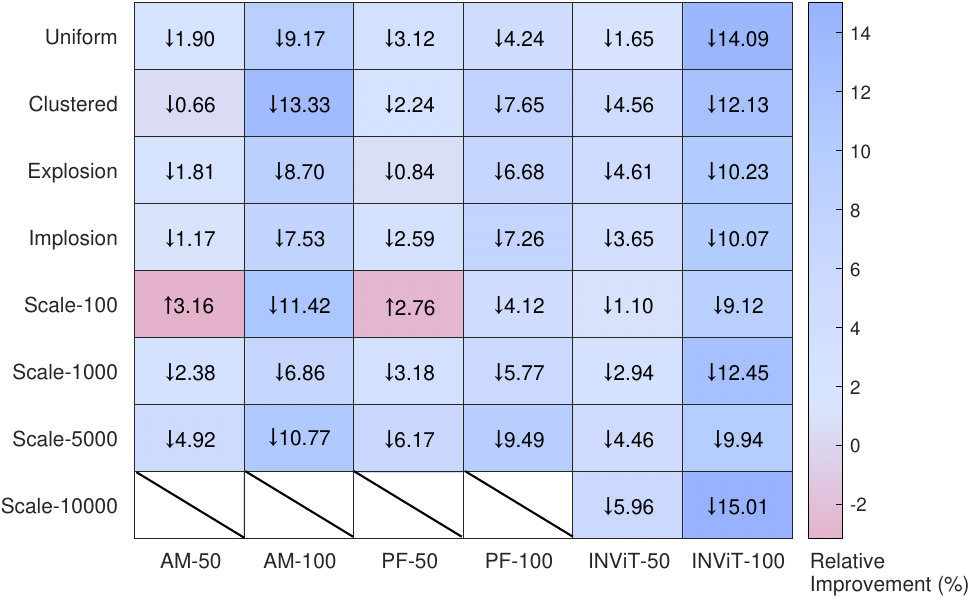}}
% \vskip -0.1in
\caption{Heatmap of relatively improved accuracy after PUPO training of each model on the randomly generated dataset. From the figure, it can be observed that PUPO achieve outstanding improvement of generalizability on most instance types.}\label{result_improve}
\end{center}
\vskip -0.2in
\end{figure}

\subsection{Generalization Performance Analysis}

\textbf{Performance on Randomly Generated Dataset.} 
Table \ref{result_rand} presents the performance of six models on a randomly generated dataset, trained using both the vanilla and PUPO methods.
%Table \ref{result_rand} presents the performance of six models on the randomly generated dataset, which are trained using both the vanilla training method and PUPO training. 
The experimental results for tour length and solving time are provided in Appendix \ref{experimental result}.
%The experimental results regarding the specific tour length and solving time are provided in Appendix \ref{experimental result}. 
It can be observed that PUPO training enhances the performance of six models across nearly all instance types.
In particular, AM-100, PF-100 and INViT-100 demonstrate superior performance across all distributions and scales compared to those trained with the vanilla. 
Notably, the PUPO-trained INViT-100 accomplishes a gap of 6.78\% on Implosion-10000, whereas the Vanilla-trained model can only reach a gap of 8.27\%.
%Notably, the PUPO-trained INViT-100 accomplishes a 6.78\% gap on Implosion-10000, whereas the Vanilla-trained model can only reach a gap of 8.27\%.

To better illustrate the improvement in generalizability with PUPO, Figure \ref{result_improve} presents a heatmap of the average relative performance improvement across various scales and distributions. 
%To more explicitly illustrate the promotion of generalizability through PUPO, Figure \ref{result_improve} shows the heatmap of average relative performance improvement of models trained with PUPO across various scales and distributions.
Combining PUPO with larger training scales leads to more significant gains. %pronounced advancements. 
Specifically, models trained at scale $100$ achieved approximately 10\% better performance, while those trained at the scale of 50 saw only a 3\% gain.
%Models trained at the scale of 100 exhibited uplifts of around 10\%, whereas those trained at the scale of 50 showed only about 3\%.

From the perspective of test scale, PUPO training significantly improves model generalization performance on larger instances.
For example, INViT-100 achieves a 15.01\% enhancement on the scale 10000, and AM-100 shows a 10.77\% gain on the scale of 5000. 
On scales closer to the training scale, diverse phenomena can be observed. 
AM-100 accomplishes an 11.42\% uplift on the scale of 100, while AM-50 and PF-50 experience negative performance changes on the same scale. 
%AM-100 accomplished an 11.42\% uplift on the scale of 100, while the improvement for AM-50 and PF-50 on the same scale is negative. 
This suggests that for models with poor expressive ability, PUPO may involve in a trade-off, sacrificing accuracy on smaller-scale instances while learning patterns more suited to larger scales. %limited expressiveness
This trade-off is linked to the implicit regularization feature of PUPO, which is discussed in the subsequent part. %further later. 
%It suggests that for models with poor expressive ability, PUPO may results in a trade-off, sacrificing accuracy on smaller-scale instances while learning patterns suitable for wide-size instances, which is due to the implicit regularization feature of PUPO analyzed in the subsequent part.
From a distributional perspective, PUPO leads to positive average relative improvements across all distributions, indicating an enhancement in generalizability with respect to distributions. %indicating enhanced generalizability with respect to distributions. 
%From the distribution perspective, the introduction of PUPO leads to positive average relative improvements across all distributions, indicating an enhancement in generalizability with respect to distributions. 

\textbf{Performance on Real-world Dataset.} 
Table \ref{result_tsplib} presents the performance on TSPLIB, further demonstrating that PUPO enhances the generalization ability on real-world data.

\begin{table}[t]
\vskip -0.1in
\caption{Performance of average gap ($\%$) on TSPLIB after training each model using both the vanilla and PUPO methods. \textcolor{blue}{Blue} highlights the decrease in the gap, while \textcolor{red}{red} indicates its increase. From the table, it can be observed that PUPO can also enhance the generalization ability on real-world dataset.}\label{result_tsplib}
    \vskip 0.05in
    % \begin{center}
    % \begin{small}
    % \begin{sc}
    \resizebox{\linewidth}{!}{
    \begin{tabular}{cc|cccccc}
    \toprule
         ~ & ~ & \multicolumn{2}{c}{1 $\sim100$} & \multicolumn{2}{c}{$101 \sim1000$} & \multicolumn{2}{c}{$1001 \sim5000$}  \\ 
        \midrule
        \multirow{2}{*}{AM-50} & Vanilla & 5.82 & & 17.39 & & 48.41 &  \\ 
        & PUPO & \textbf{5.17} & \textcolor{blue}{$\downarrow$ 0.65} & \textbf{16.14} & \textcolor{blue}{$\downarrow$ 1.25} & \textbf{46.48} & \textcolor{blue}{$\downarrow$ 1.93}   \\ 
        \cdashline{1-8} 
        \multirow{2}{*}{AM-100} & Vanilla & 4.98 & & 14.12 & & 43.33 &  \\ 
        & PUPO & \textbf{4.48} & \textcolor{blue}{$\downarrow$ 0.50} & \textbf{12.17} & \textcolor{blue}{$\downarrow$ 1.95} & \textbf{40.46} & \textcolor{blue}{$\downarrow$ 2.97}  \\ 
        \cdashline{1-8}
        \multirow{2}{*}{PF-50} & Vanilla & \textbf{7.42} & & 25.96 & & 75.57 &  \\ 
        & PUPO & 8.37 & \textcolor{red}{$\uparrow$ 0.95} & \textbf{24.77} & \textcolor{blue}{$\downarrow$ 1.19} & \textbf{72.58} & \textcolor{blue}{$\downarrow$ 2.99}  \\ 
        \cdashline{1-8} 
        \multirow{2}{*}{PF-100} & Vanilla & 7.82 & & 23.37 & & 66.05 &  \\ 
        & PUPO & \textbf{7.16}& \textcolor{blue}{$\downarrow$ 0.66} & \textbf{22.68}& \textcolor{blue}{$\downarrow$ 0.69} & \textbf{60.86}& \textcolor{blue}{$\downarrow$ 5.19}  \\ 
        \cdashline{1-8}
        \multirow{2}{*}{INViT-50} & Vanilla & \textbf{1.71} & & 4.97 & & 9.65 &  \\ 
        & PUPO & 2.02 & \textcolor{red}{$\uparrow$ 0.31} & \textbf{4.67} & \textcolor{blue}{$\downarrow$ 0.30} & \textbf{8.49} & \textcolor{blue}{$\downarrow$ 1.16}  \\ 
        \cdashline{1-8} 
        \multirow{2}{*}{INViT-100} & Vanilla & 2.43 & & 5.60 & & 9.32 &  \\ 
        & PUPO & \textbf{2.16} & \textcolor{blue}{$\downarrow$ 0.27} & \textbf{4.97} & \textcolor{blue}{$\downarrow$ 0.63} & \textbf{8.68} & \textcolor{blue}{$\downarrow$ 0.64}  \\ \bottomrule
    \end{tabular}
    }
    % \end{sc}
    % \end{small}
    % \end{center}
    \vskip -0.1in
\end{table}

Furthermore, it is evident that the integration of PUPO with competitive models yields significant improvements in generalization. %Furthermore, it is evident that the integration of PUPO with competitive models yields more widespread and significant improvements in generalization. 
As a SOTA constructive neural architecture, INViT shows notable performance boosts after PUPO training, regardless of the training scale. 
Although PUPO also enhanced the generalizability of AM and PF, their global perception and decision modules remain vulnerable, resulting in larger performance gaps as the scale increases.%, leading to a substantial growth in gaps as the scale increases. 
This observation abstracts us to focus on developing network architectures that incorporate PuLa in future work.

\subsection{Computational Efficiency Analysis}
Figure \ref{time_bar} illustrates the relative deviations in solving time between Vanilla-trained and PUPO-trained models, where positive values indicate that PUPO-trained models hold shorter solving times. 
Following the figure, the relative deviations are within 5\% across all instance types, suggesting negligible differences between the two, even the PUPO-trained models tend to solve faster in most cases. %indicating that there is almost no difference between the two, even the PUPO-trained models tend to solve faster in most cases. 
Since PUPO only provides guidance during the training phase, it does not introduce additional time overhead during the inference phase. 
Due to the tensorizable computation of PUPO, the training time does not increase significantly. 
A detailed comparison of training times can be found in Appendix \ref{experimental result}.

\begin{figure}[t]
% \vskip 0.2in
\begin{center}
\centerline{\includegraphics[width=\columnwidth]{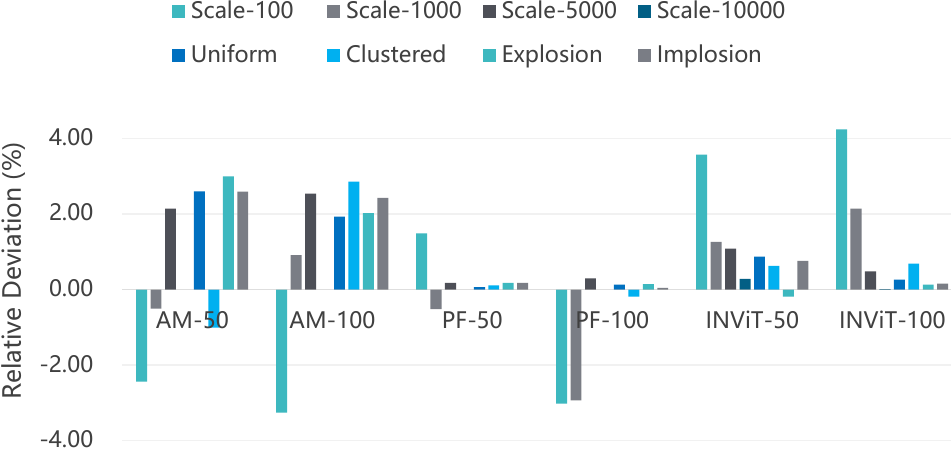}}
\caption{The bar chart of relative deviations (\%) in solving time between Vanilla- and PUPO-trained models, where positive values indicate that PUPO-trained models hold shorter solving times. From the figure, it can be observed that there is no significant difference between Vanilla- and PUPO-trained model.}
\label{time_bar}
\vskip -0.3in
\end{center}
% \vskip -0.2in
\end{figure}

%%%%%%%%%  下面再说纯净指标和隐式正则化
\subsection{Learning Mechanisms Analysis}
\textbf{Purity Perception.} 
To further explore the role of PUPO during the training process, Table \ref{purity_eval} presents three metrics that evaluate the purity of tours obtained by INViT-100. 
Numerical results of purity metrics for all models are available in Appendix \ref{experimental result}. 
The table demonstrates that PUPO-trained models consistently generate solutions with more outstanding purity. 
%Following the table, PUPO-trained models possess the ability to generate solutions with more outstanding purity. 
Compared to vanilla training, PUPO enhances the Prop-0 metric across all types, and it lead to a reduction in the two kinds of APO metric across nearly all types. 
Notably, the largest average improvement in Prop-0 occurs at the scale of 10,000, with a 1.1\% increase, which also corresponds to the scale where INViT-100 achieves its most promoted generalization performance. 
%In particular, the largest average improvement in Prop-0 occurs at the scale of 10000,  with a 1.1\% increase, which also corresponds to the scale where generalization performance of INViT-100 is promoted the most. 
These results suggest that PUPO facilitates the emergence of PuLa during training, leading to enhanced model generalization.
%It indicates that PUPO encourages the emergence of PuLa during training process, ultimately improving generalizability of models.

\begin{table}[t]
\vskip -0.1in
\caption{The experimental results of Vanilla- and PUPO-trained model on three purity evaluation metrics, where bold formatting represents superior results. From the table, it can be observed that PUPO enhances the ability of purity perception.}\label{purity_eval}
\vskip 0.05in
\resizebox{\linewidth}{!}{%
\begin{tabular}{c|cc|cc|cc}
\toprule
\multicolumn{1}{c}{} & \multicolumn{2}{|c}{Prop-0 (\%)}         & \multicolumn{2}{|c}{APO (all)}               & \multicolumn{2}{|c}{APO (non-0)} \\
\multicolumn{1}{c}{} & \multicolumn{1}{|c}{Vanilla} & \multicolumn{1}{c}{PUPO} & \multicolumn{1}{|c}{Vanilla} & \multicolumn{1}{c}{PUPO} & \multicolumn{1}{|c}{Vanilla}  & \multicolumn{1}{c}{PUPO} \\  \midrule
100                  & 80.02                       & \textbf{80.53}                    & 0.13                        & \textbf{0.12}                     & 1.16                         & \textbf{1.12}                    \\
1000                 & 85.46                       & \textbf{86.48}                    & 0.23                        & \textbf{0.19}                     & 1.81                         & \textbf{1.66}                     \\
5000                 & 43.36                       & \textbf{43.95}                    & 0.31                        & \textbf{0.24}                     & 2.52                         & \textbf{2.10}                     \\
10000                & 87.17                       & \textbf{88.27}                    & \textbf{0.54}                        & 0.55                     & \textbf{4.28}                         & 4.86                     \\
Uniform              & 74.36                       & \textbf{75.23}                    & 0.17                        & \textbf{0.14}                     & 1.44                         & \textbf{1.33}                     \\
Clustered            & 73.82                       & \textbf{74.65}                    & 0.34                        & \textbf{0.33}                     & \textbf{2.76}                         & 3.02                     \\
Explosion            & 73.91                       & \textbf{74.71}                    & 0.48                        & \textbf{0.42}                     & 3.85                         & \textbf{3.64}                     \\
Implosion            & 73.91                       & \textbf{74.64}                    & 0.21                        & \textbf{0.20}                     & \textbf{1.73}                         & 1.74        \\ \bottomrule         
\end{tabular}%
}
% \vskip -0.1in
\end{table}

\textbf{Implicit Regularization.}
Furthermore, we calculate the sum of the Frobenius norm of the parameter matrices for each model after different training, as presented in Table \ref{detailed_norm}. 
It can be observed that PUPO training reduces the sum of the parameter norms, acting as an implicit regularization mechanism. %the norms of the model parameters
In contrast, we also perform explicit regularization training by directly incorporating the sum of the norms into the reward function, but it does not result in improved generalizability. %enhance the generalizability. 
This suggests that PUPO effectively mitigates overfitting on the specific training set, encouraging the model to learn more generalizable structural patterns.
%Therefore, PUPO prevents overfitting on the specific training set, guiding the model to learn generalizable structural patterns.

\begin{table}[t]
\centering
\vskip -0.1in
\caption{The sum of Frobenius norm of the
parameter matrices for each model after different training, where bold formatting represents superior results. It can be observed that PUPO reduces the sum
of the parameter norms, acting as implicit regularization. }\label{detailed_norm}
\vskip 0.05in
\begin{tabular}{c|cccc}
\toprule
        & AM-50 & AM-100 & INViT-50 & INViT-100 \\ \midrule
Vanilla & 63.07 & 64.07  & 79.21    & 79.10     \\
PUPO    & \textbf{60.22} & \textbf{62.82}  & \textbf{79.15}    & \textbf{78.89} \\ \bottomrule    
\end{tabular}%
% \vskip -0.2in
\end{table}

\section{Conclusion}
In this paper, we reveal Purity Law~(PuLa), a fundamental structural principle for optimal TSP solutions, which defines that edge prevalence grows exponentially with the sparsity of surrounding vertices.
And we propose Purity Policy Optimization~(PUPO) to explicitly promote alignment with PuLa during the solution construction process.
Extensive experiments demonstrate that PUPO can be seamlessly integrated with popular neural solvers, significantly enhancing their generalization performance without incurring additional computational overhead during inference.
PuLa provides a novel perspective on TSP,  revealing a systematic bias toward local sparsity in global optima validated across diverse instances statistically.
In future work, on one hand, we aim to extend PuLa to general routing problems and provide a more profound theoretical explanation of its mechanism.
On the other hand, the implicit regularization properties induced by PUPO and other learning phenomenon characterized in the experiment are also worthy of further study.
What's more, we plan to explore more effective training methods for leveraging PuLa, and design network architectures that integrate PuLa to promote the perception to it, thereby achieving greater generalizability.

\section*{Impact Statement}
Due to the NP-hard nature and wide application range, TSP is fundamental in combinatorial optimization problems. 
However, the complex structure of problem makes it challenging to conclude the general properties of TSP. Compared to traditional time-consuming methods, neural approaches are considered to possess the potential for fast and high-quality solutions. But existing learning-based methods suffer from poor generalizability. Therefore, it is meaningful to uncover the universal and stable patterns present in optimal TSP solutions and apply them to guide the learning process of models, thereby improving their generalizability without additional computational overhead.
In addition to improving the related machine learning fields, the contribution can be further extended to industries like logistics, circuit compilation or computational biology.

% In the unusual situation where you want a paper to appear in the
% references without citing it in the main text, use \nocite
\nocite{langley00}

\bibliography{example_paper}
\bibliographystyle{icml2025}

%%%%%%%%%%%%%%%%%%%%%%%%%%%%%%%%%%%%%%%%%%%%%%%%%%%%%%%%%%%%%%%%%%%%%%%%%%%%%%%
%%%%%%%%%%%%%%%%%%%%%%%%%%%%%%%%%%%%%%%%%%%%%%%%%%%%%%%%%%%%%%%%%%%%%%%%%%%%%%%
% APPENDIX
%%%%%%%%%%%%%%%%%%%%%%%%%%%%%%%%%%%%%%%%%%%%%%%%%%%%%%%%%%%%%%%%%%%%%%%%%%%%%%%
%%%%%%%%%%%%%%%%%%%%%%%%%%%%%%%%%%%%%%%%%%%%%%%%%%%%%%%%%%%%%%%%%%%%%%%%%%%%%%%
\newpage
\appendix
\onecolumn

\section{Details of the Radar Chart Construction}\label{radar_process}
We select AM-100 and INViT-100 to represent the traditional and advanced solvers, respectively. The average gaps for each solver are computed across four scales and five distributions, which are listed in Table \ref{result_radar}. The reciprocal of the gap is then taken, and the values are normalized to the maximum value within each dimension. These derived metrics are used to plot the radar chart.

\begin{table}[h]
\caption{The Experimental results used to plot the radar chart.}\label{result_radar}
\vskip 0.15in
\resizebox{\textwidth}{!}{%
\begin{tabular}{cc|ccccccccc}
\toprule
                           &         & Scale-100 & Scale-1000 & Scale-5000 & Scale-10000 & Uniform & Clustered & Explosion & Implosion & TSPLIB \\ \midrule
\multirow{2}{*}{AM-100}    & Vanilla & 5.57      & 29.75      & 69.55      & 100.00      & 38.98   & 40.67     & 37.52     & 36.81     & 20.81  \\
                           & PUPO    & 4.90      & 27.46      & 62.36      & 86.00       & 34.40   & 36.65     & 33.74     & 34.05     & 19.03  \\
\multirow{2}{*}{INViT-100} & Vanilla & 2.82      & 7.76       & 8.85       & 9.30        & 6.73    & 8.02      & 8.68      & 7.07      & 5.78   \\
                           & PUPO    & 2.55      & 6.82       & 7.73       & 7.94        & 5.73    & 7.02      & 7.80      & 6.31      & 5.27 \\ \bottomrule 
\end{tabular}%
}
\end{table}

% \section{Problem Definition}\label{Problem Definition}
% In this section, we provide the detailed definition of the Euclidean Traveling Salesman Problem in two-dimensional space. Let's denote an undirected graph as $G = (X, E)$, where $X = \{x_i|1\leq i\leq N \}$ is the set of vertices, $E = \{e_{ij}|1\leq i, j\leq N\}$
% represents the set of edges and $N$ is the number of nodes. For Euclidean TSP, $G$ is fully connected and symmetric.
% In graph $G$, vertex $x_i \in X$ represents the vertex in the TSP instance whose coordinate is $(x_i^1, x_i^2)$. For every edge $e_{ij}$, there is a traverse cost $c(x_i, x_j)$ between $x_i$ and $x_j$, which is equal to the Euclidean distance between them in general. A Hamiltonian cycle $\{ x_{\tau_1}, x_{\tau_2}, \cdots, x_{\tau_N} \}$ of $G$ is defined as a tour that visits all vertices in $G$ exactly once. The goal of TSP is to find the one with the minimum total cost from all feasible Hamiltonian cycles. The total cost of a solution $\omega$ can be formulated as follows:
% \begin{equation}
% L(\omega) =  c(x_{\tau_N}, x_{\tau_1}) + \sum_{i=2}^{N}c(x_{\tau_{i-1}}, x_{\tau_{i}}), 
% \nonumber
% \end{equation}
% where $x_{\tau_{i}}$ is the $i$-th node in the feasible Hamiltonian cycle $\tau$. Without loss of generality, it can be assumed that all coordinates are in $[0, 1]$.

\section{Topological Properties of $0$-order Pure Edges}\label{proof topo}

As the class of edges with the lowest degree of redundancy, 0-order pure edges are demonstrated by the following propositions to possess a series of favorable topological properties.

\begin{proposition}\label{exist}
    Given an instance $\mathcal{X}$, for any vertex $x \in \mathcal{X}$, there exists at least one vertex that can form a 0-order pure edge with $x$.
\end{proposition}

\begin{proof}
    We prove the proposition by contradiction. 
    
    Assume that for any \( x \in \mathcal{X} \), we choose 
    $$ y(x) := \arg \min \|x - y\|^2 .$$
    Since no vertex can form a $0$-order pure edge with \( x \), we have that
    $$ K_p(x, y(x)) > 0, \quad \forall x \in \mathcal{X} .$$
    Therefore, there must exist a point \( z(x) \) such that 
    $$ (x-z(x))^T (y(x)-z(x))<0 .$$
    Then we can derive that
    \begin{equation}\notag
        \begin{aligned}
            &  \|x-y(x)\|^2 \\
            = & \|x-z(x)+z(x)-y(x)\|^2 \\
            = & \|x-z(x)\|^2 +\|z(x)-y(x)\|^2 - 2(x-z(x))^T (y(x)-z(x)) \\
            \ge & \|x-z(x)\|^2 +\|z(x)-y(x)\|^2 \\
            \ge & \|x-z(x)\|^2,            
        \end{aligned}
    \end{equation}
    which means that the distance between \( z(x) \) and \( x \) is smaller than the distance between \( y(x) \) and \( x \). Therefore, it contradicts the assumption that \( y(x) \) is the nearest neighbor of \( x \).
\end{proof}

\begin{proposition}\label{connected}
    Given an instance $\mathcal{X}$, the subgraph $G_0=(\mathcal{X}, \mathcal{E}_0)$ is connected, where $\mathcal{E}_0$ is the edge set of all 0-order pure edges.
\end{proposition}

\begin{proof}
    We prove the proposition by contradiction. 
    
    Suppose the subgraph \( G_0 \) is not connected. Without loss of generality, assume that \( G_0 \) has two connected components, denoted as \( F_1 \) and \( F_2 \). Let \( f_1 \in F_1 \) and \( f_2 \in F_2 \) be such that
    $$ \text{dist}(f_1, f_2) = \text{dist}(F_1, F_2) = \min_{a \in F_1} \min_{b \in F_2} \text{dist}(a, b) .$$
    Since \( e_{f_1 f_2} \notin G_0 \), we have that
    $$ K_p(f_1, f_2) > 0 .$$
    Therefore, there must exist a point \( f_3 \) such that the following inequality holds:
    $$ (f_1-f_3)^T (f_2-f_3)<0 ,$$
    which means that the distance between \( f_3 \) and \( f_1 \) is smaller than the distance between \( f_2 \) and \( f_1 \). Regardless of whether \( f_3 \in F_1 \) or \( f_3 \in F_2 \), this leads to a contradiction with the assumption that the distance between \( f_1 \) and \( f_2 \) is the shortest distance between \( F_1 \) and \( F_2 \).

\end{proof}

\begin{proposition}\label{convex}
    Given an instance $\mathcal{X}$, for any vertex $x \in \mathcal{X}$, the polyhedron formed by its 0-order pure neighbors $D_0^x$ is convex.
\end{proposition}

\begin{proof}
    Consider any three adjacent 0-order pure neighbors \( A \), \( B \), and \( C \) of point \( X \). We have that
    $$ \angle ABC = \angle ABX + \angle XBC .$$
    Since \( A \) and \( C \) are 0-order pure neighbors of \( X \), it follows that 
    $$ \angle ABX, \angle XBC < \frac{\pi}{2} ,$$
    and thus \( \angle ABC < \pi \). 
    
    Due to the arbitrariness of \( A \), \( B \), and \( C \), the polyhedron formed by its 0-order pure neighbors $D_0^x$ is convex.
\end{proof}

Proposition \ref{exist} establishes the existence of 0-order pure neighbors for any point, providing a foundation for subsequent analysis. 
While proposition \ref{connected} demonstrates that the subgraph formed by 0-order pure edges possesses overall graph structural properties. 
And proposition \ref{convex} describes the intrinsic topological features of the 0-order neighbor set for any given point. 
The three propositions above characterize the mathematical properties of purity order theoretically. 

\section{Dataset Description}\label{stat dataset}
To construct the dataset in statistical experiments, we first generate 21 different scales within the range of 20 to 1000, with intervals of 50 except 20. For each scale, we consider four widely recognized classical distributions, which is uniform, cluster, explosion and implosion, resulting in 84 different instance types totally.
For instance types with scales under 500, 256 instances are randomly sampled to form the dataset, while for those with sizes of 500 or above, the number of instances is reduced to 128, which is due to the fact that optimal solutions of large-scale instances require an excessive amount of time.
The optimal solutions for all instances are solved by LKH-3~\cite{helsgaun2000effective,helsgaun2017extension}, which is the SOTA heuristic capable of producing optimal solutions even for large-scale instances.

\section{Detailed Fitting Result}\label{fitting result}
The detailed values of the fitting error, $\alpha$ and $\beta$ for each instance type are listed in Table \ref{tab_fitting_result_whole}. 
According to the table, the low mean and variance of the fitting errors demonstrate the reliability and university of the fitting results across different instance scales and distributions.

\begin{table}[h]
\caption{The fitting errors, coefficients $\alpha$ and $\beta$ in Sec.~\ref{subsection: purity law}}\label{tab_fitting_result_whole}
\vskip 0.15in
\resizebox{\linewidth}{!}{%
\begin{tabular}{c|cccc|cccc|cccc}
\toprule
     & \multicolumn{4}{c|}{Fitting Error}            & \multicolumn{4}{c|}{$\alpha$}                       & \multicolumn{4}{c}{$\beta$}                       \\
     & Uniform  & Clustered & Explosion & Implosion & Uniform & Clustered & Explosion & Implosion & Uniform & Clustered & Explosion & Implosion \\ \midrule
20   & 6.58E-05 & 2.45E-04  & 1.01E-04  & 5.71E-05  & 0.8860   & 0.8554    & 0.8835    & 0.8872    & 2.2514  & 2.0800      & 2.2483    & 2.2630     \\
50   & 5.42E-05 & 7.90E-05  & 3.96E-05  & 5.05E-05  & 0.9115  & 0.8919    & 0.9027    & 0.9002    & 2.5166  & 2.3295    & 2.4152    & 2.3923    \\
100  & 2.41E-05 & 3.75E-05  & 1.85E-05  & 3.61E-05  & 0.9204  & 0.9124    & 0.9091    & 0.9137    & 2.5974  & 2.5133    & 2.4567    & 2.5286    \\
150  & 2.54E-05 & 2.57E-05  & 2.96E-05  & 1.78E-05  & 0.9236  & 0.9188    & 0.9150     & 0.9184    & 2.6410   & 2.5782    & 2.5412    & 2.5681    \\
200  & 2.18E-05 & 2.20E-05  & 2.62E-05  & 1.63E-05  & 0.9236  & 0.9221    & 0.9197    & 0.9202    & 2.6384  & 2.6192    & 2.6002    & 2.5887    \\
250  & 1.18E-05 & 1.35E-05  & 2.38E-05  & 2.42E-05  & 0.9263  & 0.9245    & 0.9209    & 0.9227    & 2.6568  & 2.6366    & 2.6066    & 2.6337    \\
300  & 1.26E-05 & 1.49E-05  & 2.06E-05  & 2.08E-05  & 0.9259  & 0.9258    & 0.9211    & 0.9224    & 2.6535  & 2.6578    & 2.6057    & 2.6201    \\
350  & 1.39E-05 & 1.40E-05  & 1.61E-05  & 1.71E-05  & 0.9273  & 0.9263    & 0.9236    & 0.9224    & 2.6762  & 2.6621    & 2.6318    & 2.6166    \\
400  & 1.33E-05 & 1.54E-05  & 1.80E-05  & 1.05E-05  & 0.9278  & 0.9280     & 0.9241    & 0.9235    & 2.6814  & 2.6918    & 2.6430     & 2.6206    \\
450  & 1.04E-05 & 1.12E-05  & 1.61E-05  & 1.52E-05  & 0.9280   & 0.9288    & 0.9234    & 0.9260     & 2.6805  & 2.6947    & 2.6289    & 2.6633    \\
500  & 8.38E-06 & 1.15E-05  & 1.18E-05  & 1.28E-05  & 0.9273  & 0.9288    & 0.9256    & 0.9255    & 2.6634  & 2.6936    & 2.6515    & 2.6536    \\
550  & 1.04E-05 & 1.49E-05  & 1.94E-05  & 1.31E-05  & 0.9306  & 0.9298    & 0.9257    & 0.9267    & 2.7190   & 2.7186    & 2.6615    & 2.6673    \\
600  & 7.55E-06 & 1.08E-05  & 1.48E-05  & 1.26E-05  & 0.9294  & 0.9306    & 0.9259    & 0.9275    & 2.6931  & 2.7207    & 2.6591    & 2.6803    \\
650  & 1.13E-05 & 1.37E-05  & 9.63E-06  & 1.79E-05  & 0.9297  & 0.9310     & 0.9262    & 0.9263    & 2.7058  & 2.7293    & 2.6535    & 2.6689    \\
700  & 1.34E-05 & 1.02E-05  & 1.40E-05  & 1.77E-05  & 0.9280   & 0.9306    & 0.9277    & 0.9256    & 2.6856  & 2.7175    & 2.6877    & 2.6563    \\
750  & 1.22E-05 & 1.19E-05  & 1.50E-05  & 1.26E-05  & 0.9287  & 0.9308    & 0.9264    & 0.9282    & 2.6959  & 2.7251    & 2.6658    & 2.6930     \\
800  & 1.22E-05 & 1.29E-05  & 9.97E-06  & 1.19E-05  & 0.9294  & 0.9311    & 0.9255    & 0.9273    & 2.7040   & 2.7323    & 2.6459    & 2.6750     \\
850  & 1.30E-05 & 1.12E-05  & 1.80E-05  & 1.33E-05  & 0.9300    & 0.9305    & 0.9276    & 0.9271    & 2.7151  & 2.7184    & 2.6909    & 2.6707    \\
900  & 1.33E-05 & 1.03E-05  & 1.08E-05  & 1.46E-05  & 0.9309  & 0.9316    & 0.9276    & 0.9281    & 2.7296  & 2.7357    & 2.6790     & 2.6906    \\
950  & 1.33E-05 & 1.38E-05  & 1.44E-05  & 1.41E-05  & 0.9299  & 0.9314    & 0.9284    & 0.9277    & 2.7132  & 2.7376    & 2.6951    & 2.6855    \\
1000 & 1.51E-05 & 1.21E-05  & 1.26E-05  & 1.24E-05  & 0.9310   & 0.9318    & 0.9282    & 0.9289    & 2.7340   & 2.7413    & 2.6890     & 2.6978   \\ \bottomrule
\end{tabular}%
}
\vskip -0.1in
\end{table}

\section{Proof of the supermodularity of Purity Availability $\phi$}\label{proof_subm}

\begin{proposition}\label{subm}
    The set function $\phi: 2^X \to \mathbb{R} $ defined on the subsets of the finite set $X$ is supermodular, that is, for any subset $A \subseteq B \subseteq X$ and any $x \in X \setminus B$, the following inequality holds:
    \[
    \phi(A \cup \{x\}) - \phi(A) \leq \phi(B \cup \{x\}) - \phi(B).
    \]
\end{proposition}

\begin{proof}
Given subsets $U \subseteq X$ and any $v_1, v_2 \in X \setminus U$, we prove the equivalent definition of submodular functions:
\begin{equation}
    \phi(U \cup \{v_1\}) + \phi(U \cup \{v_2\}) \leq \phi(U \cup \{v_1, v_2\}) - \phi(U).
\end{equation}
First, we have
\begin{align*}
      \phi (U) &= \frac{\sum \limits_{x_i \in U} \mathop{\min} \limits_{\substack{x_j \in U \\ j \neq i}} K_{p}(x_i, x_j)}{|U|}. \\
      \phi (U \cup \{v_1\}) &= \frac{\sum \limits_{x_i \in U} \mathop{\min} \limits_{\substack{x_j \in {U \cup \{v_1\}} \\ j \neq i}} K_{p}(x_i, x_j) + \mathop{\min} \limits_{\substack{x_j \in U}} K_{p}(v_1, x_j)}{|U|+1}   \\
      \phi (U \cup \{v_2\}) &= \frac{\sum \limits_{x_i \in U} \mathop{\min} \limits_{\substack{x_j \in {U \cup \{v_2\}} \\ j \neq i}} K_{p}(x_i, x_j) + \mathop{\min} \limits_{\substack{x_j \in U}} K_{p}(v_2, x_j)}{|U|+1}   \\
      \phi (U \cup \{v_1, v_2\}) &= \frac{\sum \limits_{x_i \in U} \mathop{\min} \limits_{\substack{x_j \in {U \cup \{v_1, v_2\}} \\ j \neq i}} K_{p}(x_i, x_j) 
      + \mathop{\min} \limits_{\substack{x_j \in {U \cup \{v_2\}}}} K_{p}(v_1, x_j) 
      + \mathop{\min} \limits_{\substack{x_j \in {U \cup \{v_1\}}}} K_{p}(v_2, x_j)}{|U|+2}.   \\
\end{align*}
By the following relations, we partition $U$ into three parts, denoted as $U_0, U_1, U_2$, respectively. 
\begin{align*}
 U_0 &= \{x_i \mid \arg \min_{y} K_p(x_i, y) \in U\} \\
U_1 &= \{x_i \mid \arg \min_{y} K_p(x_i, y) = v_1\} \\
U_2 &= \{x_i \mid \arg \min_{y} K_p(x_i, y) = v_2\} 
\end{align*}
Then, we prove that the following expression is non-positive from three parts.
\[\left( \phi(U \cup \{v_1\}) + \phi(U \cup \{v_2\})\right) - \left(\phi(U \cup \{v_1, v_2\}) + \phi(U)\right).\]
% For convenience, we denote $|U|(|U|+1)(|U|+2)$ as $\bigtriangleup$.

\textbf{Part 1}
\begin{align*}
&\sum_{x_i \in U_0} \min_{x_j \in U} K_p(x_i, x_j) \left[ \frac{2}{|U|+1} - \left( \frac{1}{|U|} + \frac{1}{|U|+2} \right) \right] \\
\leq &\sum_{x_i \in U_0} \min_{x_j \in U} K_p(x_i, x_j) \left( \frac{-1}{|U|(|U|+1)(|U|+2)} \right) \leq 0
\end{align*}

\textbf{Part 2}
\begin{align*}
&\sum_{x_i \in U_1} \left( \left( \frac{\min_{x_j \in U}K_p(x_i, x_j)}{|U|+1}  
-  \frac{\min_{x_j \in U}K_p(x_i, x_j)}{|U|} \right)
+\left( 
\frac{K_p(x_i, v_1)}{|U|+1}
-  \frac{K_p(x_i, v_1)}{|U|+2} \right) \right) 
\\
&+  \left( \frac{\min_{x_j \in U}K_p(x_j, v_1)}{|U|+1}
-  \frac{\min_{x_j \in U \cup \{v_2\}}K_p(x_j, v_1)}{|U|+2} \right)\\
= & \sum_{x_i \in U_1} \left( -
\frac{\min_{x_j \in U}K_p(x_i, x_j)}{|U|(|U|+1)}  
+  \frac{K_p(x_i, v_1)}{(|U|+1)(|U|+2)} \right)  
+  \frac{\min_{x_j \in U}K_p(x_j, v_1)}{(|U|+1)(|U|+2)}
 \\
 \leq & \sum_{x_i \in U_1} \left( -
\frac{K_p(x_i, v_1)+1}{|U|(|U|+1)}  
+  \frac{K_p(x_i, v_1)}{(|U|+1)(|U|+2)} \right)  
+  \frac{\min_{x_j \in U}K_p(x_j, v_1)}{(|U|+1)(|U|+2)}\\
= & \frac{1}{|U|(|U|+1)(|U|+2)}\left(
-2 \sum_{x_j \in U_1}K_p(x_i, v_1) 
+ |U|\min_{x_j \in U}K_p(x_j, v_1)
- |U_1|(|U|+2)
   \right)\\
\leq & \frac{1}{|U|(|U|+1)(|U|+2)}\left(
-2|U_1|\min_{x_j \in U}K_p(x_j, v_1) 
+ |U|\min_{x_j \in U}K_p(x_j, v_1)
- |U_1|(|U|+2)
   \right)\\
\leq & \frac{1}{|U|(|U|+1)(|U|+2)}\left(
-2|U_1|^2 - 2|U_1|
   \right)
\leq 0
\end{align*}

\textbf{Part 3}
\begin{align*}
&\sum_{x_i \in U_2} \left( \left( \frac{\min_{x_j \in U}K_p(x_i, x_j)}{|U|+1}  
-  \frac{\min_{x_j \in U}K_p(x_i, x_j)}{|U|} \right)
+\left( 
\frac{K_p(x_i, v_2)}{|U|+1}
-  \frac{K_p(x_i, v_2)}{|U|+2} \right) \right) 
\\
&+  \left( \frac{\min_{x_j \in U}K_p(x_j, v_2)}{|U|+1}
-  \frac{\min_{x_j \in U \cup \{v_1\}}K_p(x_j, v_2)}{|U|+2} \right)\\
= & \sum_{x_i \in U_2} \left( -
\frac{\min_{x_j \in U}K_p(x_i, x_j)}{|U|(|U|+1)}  
+  \frac{K_p(x_i, v_2)}{(|U|+1)(|U|+2)} \right)  
+  \frac{\min_{x_j \in U}K_p(x_j, v_2)}{(|U|+1)(|U|+2)}
 \\
 \leq & \sum_{x_i \in U_2} \left( -
\frac{K_p(x_i, v_2)+1}{|U|(|U|+1)}  
+  \frac{K_p(x_i, v_2)}{(|U|+1)(|U|+2)} \right)  
+  \frac{\min_{x_j \in U}K_p(x_j, v_2)}{(|U|+1)(|U|+2)}\\
= & \frac{1}{|U|(|U|+1)(|U|+2)}\left(
-2 \sum_{x_j \in U_2}K_p(x_i, v_2) 
+ |U|\min_{x_j \in U}K_p(x_j, v_2)
- |U_2|(|U|+2)
   \right)\\
\leq & \frac{1}{|U|(|U|+1)(|U|+2)}\left(
-2|U_2|\min_{x_j \in U}K_p(x_j, v_2) 
+ |U|\min_{x_j \in U}K_p(x_j, v_2)
- |U_2|(|U|+2)
   \right)\\
\leq & \frac{1}{|U|(|U|+1)(|U|+2)}\left(
-2|U_2|^2 - 2|U_2|
   \right)
\leq 0
\end{align*}

The gain consists of three parts, each of which is non-positive. Therefore, the overall gain is non-positive, and the supermodular property is thus proven.

\end{proof}

\section{The Whole PUPO Algorithm}\label{PUPO_alg}
The whole algorithm of PUPO is presented in Algorithm \ref{alg:PUPO Training}. PUPO is established upon the REINFORCE with baseline, incorporating PuLa information into the modified policy gradient.
At different stages of solution construction, PUPO introduces discounted cost information that reflects the purity potential of the partial solution and the unvisited vertex set, which encourages the model to take actions with lower purity orders at each step during the training process. 
This approach helps the model to learn consistent patterns with the PuLa prior, an information that is independent of specific instances, thereby enhancing its generalization ability.
Notably, PUPO can be easily integrated with arbitrary existing constructive neural solvers, without any alterations to network architecture.
\begin{algorithm}[tb]
\caption{PUPO Training}\label{alg:PUPO Training}
% \vskip 0.15in
\begin{algorithmic}
\STATE \textbf{Input:} number of epochs \( E \), steps per epoch \( M \), batch size \( B \), scale of training instances \( N \), discount factor  \( \gamma \), learning rate \( \delta \)
\STATE Init \( \theta \), \( \theta^{BL} \leftarrow \theta \)
\FOR{epoch = 1, \dots, \( E \)}
    \FOR{step = 1, \dots, \( M \)}
        \STATE \(\forall i \in \{1, 2, \dots, B\} \)
        \STATE \( s_i \leftarrow \text{RandomInstance()} \quad  \)
        \STATE \( \boldsymbol{\tau}^i \leftarrow \text{SampleRollout}(s_i, p_{\theta}) \quad  \)
        \STATE \( \boldsymbol{\tau}^{i, BL} \leftarrow \text{GreedyRollout}(s_i, p_{\theta^{BL}}) \quad  \)
        \STATE \( U_{0}^i \leftarrow s_i \)
        \FOR{time = 1, \dots, \( N-1 \)}
            \STATE \( U_{t}^i \leftarrow U_{t-1}^i \setminus {\tau_{t}^i } \)
            \STATE \( \phi (U_t^i) \leftarrow \sum \limits_{a \in U_t^i} \mathop{\min} \limits_{\substack{b \in U_t^i \\ b \neq a}} K_{p}(a, b, s_i) \)
            \STATE 
            $\begin{aligned}C(U_t^i, \tau_{t+1}^i) \leftarrow &K_p(\tau_t^i, \tau_{t+1}^i, s_i) + \frac{\phi (U_{t+1}^i)}{|U_{t+1}^i|} - \frac{\phi (U_{t}^i)}{|U_{t}^i|} 
            \end{aligned} $
            \STATE \( W_{t+1}^i  \leftarrow 1 + \sum \limits_{j=t}^{N} \gamma^{j-t} C(U_j^i, \tau_{j+1}^i) \quad   \)
        \ENDFOR
        \STATE 
        $\begin{aligned} 
         \nabla \hat{\mathcal{L}} \leftarrow \frac{1}{B} \sum_{i=1}^{B} &\left( L(\boldsymbol{\tau}^i) - L(\boldsymbol{\tau}^{i, BL}) \right) \cdot  \left( \sum \limits_{t=2}^{N} W_{t}^i \nabla_\theta \log p_{\theta}(\tau_{t}^i | \tau_{1:t-1}^i, s_i) \right)
        \end{aligned} $
        \STATE \( \theta \leftarrow \theta + \delta
        \nabla \hat{\mathcal{L}} \)
    \ENDFOR

    \STATE \( \theta^{BL} \leftarrow \text{UpdateBaseline}(\theta, \theta^{BL}) \)

\ENDFOR
\end{algorithmic}
% \vskip -0.1in
\end{algorithm}

\section{Detailed Experimental Result}\label{experimental result}
Table \ref{detailed_train_time} shows the training time per epoch with different methods. We can see that the training time of PUPO does not increase significantly owing to the tensorizable computation.
The experimental results of the specific tour length is presented in Table \ref{detailed_length}.
Table \ref{detailed_time} illustrates the solving time of Vanilla-trained and PUPO-trained models on randomly generated dataset. Following the table, it can be observed that there is almost no
difference between Vanilla and PUPO.
The numerical results of purity metrics for all models are presented in Table \ref{detailed_purity}. Following the table, PUPO-trained models possess the ability to generate solutions with more outstanding purity.

\begin{table}[!ht]
\centering
\caption{The numerical results of training time (min) per epoch during different training.}\label{detailed_train_time}
\vskip 0.15in
    \begin{tabular}{c|cccccc}
    \toprule
        ~ & AM-50 & AM-100 & INViT-50 & INViT-100   \\ \midrule
        Vanilla & 3.05  & 4.49  & 7.72  & 15.75    \\ 
        PUPO & 5.95  & 14.20  & 11.598 & 21.40   \\ \bottomrule
    \end{tabular}
    \vskip -0.1in
\end{table}

\begin{table}[]
\centering
\caption{The length of tours generated from each model after different training on randomly generated dataset.}\label{detailed_length}
\vskip 0.15in
\resizebox{\textwidth}{!}{%
\begin{tabular}{cc|cccc|cccc|cccc}
\toprule
          &       & \multicolumn{2}{c}{AM-50} & \multicolumn{2}{c|}{AM-100} & \multicolumn{2}{c}{PF-50} & \multicolumn{2}{c|}{PF-100} & \multicolumn{2}{c}{INViT-50} & \multicolumn{2}{c}{INViT-100} \\
          \midrule
          &       & Vanilla      & PUPO       & Vanilla       & PUPO       & Vanilla      & PUPO       & Vanilla       & PUPO       & Vanilla        & PUPO        & Vanilla        & PUPO         \\
Uniform   & 100   & 8.28         & 8.29       & 8.27          & 8.23       & 8.19         & 8.21       & 8.15          & 8.13       & 8.03           & 8.04        & 8.07           & 8.05         \\
          & 1000  & 31.47        & 31.12      & 30.47         & 30.22      & 32.74        & 32.30      & 30.32         & 30.17      & 24.91          & 24.88       & 24.94          & 24.67        \\
          & 5000  & 85.75        & 84.28      & 92.20         & 85.97      & 109.00       & 104.18     & 92.18         & 90.62      & 55.85          & 55.69       & 55.84          & 55.15        \\
          & 10000 & -            & -          & -             & -          & -            & -          & -             & -          & 79.07          & 78.88       & 78.79          & 77.87        \\
Clustered & 100   & 5.72         & 5.73       & 5.71          & 5.63       & 5.71         & 5.71       & 5.68          & 5.67       & 5.46           & 5.46        & 5.50           & 5.48         \\
          & 1000  & 19.82        & 19.65      & 19.39         & 18.80      & 21.26        & 21.05      & 20.78         & 20.30      & 15.25          & 15.19       & 15.24          & 15.07        \\
          & 5000  & 53.06        & 52.50      & 52.64         & 50.77      & 70.86        & 69.10      & 70.71         & 66.13      & 32.79          & 32.66       & 32.73          & 32.39        \\
          & 10000 & -            & -          & -             & -          & -            & -          & -             & -          & 44.88          & 44.58       & 44.72          & 44.17        \\
Explosion & 100   & 6.86         & 6.87       & 6.86          & 6.83       & 6.86         & 6.87       & 6.84          & 6.83       & 6.67           & 6.66        & 6.70           & 6.68         \\
          & 1000  & 21.90        & 21.84      & 21.47         & 21.03      & 23.50        & 23.36      & 22.72         & 22.10      & 17.67          & 17.60       & 17.67          & 17.53        \\
          & 5000  & 57.30        & 55.81      & 58.14         & 55.62      & 77.35        & 75.30      & 72.57         & 69.02      & 37.52          & 37.28       & 37.23          & 36.88        \\
          & 10000 & -            & -          & -             & -          & -            & -          & -             & -          & 43.28          & 43.05       & 43.18          & 42.66        \\
Implosion & 100   & 7.48         & 7.49       & 7.48          & 7.45       & 7.45         & 7.47       & 7.43          & 7.42       & 7.29           & 7.29        & 7.32           & 7.30         \\
          & 1000  & 27.61        & 27.57      & 26.63         & 26.33      & 29.39        & 29.04      & 27.28         & 26.95      & 21.64          & 21.63       & 21.69          & 21.51        \\
          & 5000  & 72.59        & 71.24      & 71.28         & 68.46      & 100.18       & 95.14      & 86.94         & 80.27      & 44.96          & 44.83       & 45.05          & 44.78        \\
          & 10000 & -            & -          & -             & -          & -            & -          & -             & -          & 72.24          & 71.79       & 72.17          & 71.18       \\
\bottomrule    
\end{tabular}%
}
\vskip -0.1in
\end{table}

\begin{table}[]
\centering
\caption{The solving time for each model with different training methods on randomly generated dataset.}\label{detailed_time}
\vskip 0.15in
\resizebox{\textwidth}{!}{%
\begin{tabular}{c|cccc|cccc|cccc}
\toprule
                & \multicolumn{2}{c}{AM-50} & \multicolumn{2}{c|}{AM-100} & \multicolumn{2}{c}{PF-50} & \multicolumn{2}{c|}{PF-100} & \multicolumn{2}{c}{INViT-50} & \multicolumn{2}{c}{INViT-100} \\ \midrule
                & Vanilla      & PUPO       & Vanilla       & PUPO       & Vanilla      & PUPO       & Vanilla      & PUPO        & Vanilla       & PUPO         & Vanilla        & PUPO         \\
Uniform-100     & 0.20         & 0.22       & 0.21          & 0.22       & 2.02         & 1.97       & 1.00         & 1.02        & 0.85          & 0.80         & 1.07           & 0.91         \\
Uniform-1000    & 2.35         & 2.34       & 2.41          & 2.40       & 23.67        & 24.08      & 13.57        & 14.18       & 13.24         & 12.88        & 17.83          & 17.34        \\
Uniform-5000    & 19.68        & 19.10      & 20.00         & 19.57      & 246.65       & 246.12     & 205.14       & 204.15      & 91.20         & 89.44        & 120.07         & 119.87       \\
Uniform-10000   & -            & -          & -             & -          & -            & -          & -            & -           & 224.99        & 223.69       & 288.44         & 287.11       \\
Cluster-100     & 0.21         & 0.21       & 0.21          & 0.21       & 2.08         & 2.13       & 0.97         & 0.98        & 0.84          & 0.80         & 1.04           & 1.17         \\
Cluster-1000    & 2.35         & 2.35       & 2.44          & 2.43       & 25.28        & 26.11      & 14.05        & 15.28       & 13.00         & 12.77        & 17.57          & 16.56        \\
Cluster-5000    & 19.16        & 19.38      & 19.99         & 19.35      & 275.68       & 274.49     & 212.36       & 211.68      & 89.48         & 88.23        & 118.63         & 117.31       \\
Cluster-10000   & -            & -          & -             & -          & -            & -          & -            & -           & 222.47        & 222.32       & 288.67         & 288.90       \\
Explosion-100   & 0.21         & 0.21       & 0.20          & 0.22       & 2.18         & 2.17       & 0.97         & 1.09        & 0.82          & 0.83         & 1.05           & 0.94         \\
Explosion-1000  & 2.38         & 2.36       & 2.46          & 2.40       & 25.44        & 24.53      & 13.94        & 13.82       & 12.94         & 13.05        & 17.70          & 17.98        \\
Explosion-5000  & 19.30        & 18.67      & 19.98         & 19.57      & 411.08       & 411.22     & 213.01       & 212.39      & 89.58         & 89.93        & 118.93         & 118.34       \\
Explosion-10000 & -            & -          & -             & -          & -            & -          & -            & -           & 223.24        & 223.03       & 288.86         & 288.32       \\
Implosion-100   & 0.21         & 0.20       & 0.21          & 0.21       & 3.02         & 2.89       & 1.00         & 0.97        & 0.83          & 0.79         & 1.05           & 1.01         \\
Implosion-1000  & 2.37         & 2.45       & 2.45          & 2.45       & 20.84        & 21.01      & 14.08        & 13.99       & 12.91         & 12.73        & 17.60          & 17.31        \\
Implosion-5000  & 19.95        & 19.28      & 19.84         & 19.30      & 445.75       & 444.87     & 186.14       & 186.05      & 89.92         & 88.30        & 118.87         & 118.69       \\
Implosion-10000 & -            & -          & -             & -          & -            & -          & -            & -           & 222.56        & 221.72       & 289.39         & 290.99      \\
\bottomrule    
\end{tabular}%
}
\vskip -0.1in
\end{table}

\begin{table}[]
\centering
\caption{The numerical results of purity evaluation for each model after different training.}\label{detailed_purity}
\vskip 0.15in
\resizebox{\textwidth}{!}{%
\begin{tabular}{c|cccccc|ccccccc}
\toprule
                & \multicolumn{6}{c}{AM-50}                                                     & \multicolumn{6}{c}{AM-100}                                                    &  \\
                & \multicolumn{3}{c}{Vanilla}           & \multicolumn{3}{c}{PUPO}              & \multicolumn{3}{c}{Vanilla}           & \multicolumn{3}{c}{PUPO}              &  \\
                & Prop-0 (\%) & APO (all) & APO (non-0) & Prop-0 (\%) & APO (all) & APO (non-0) & Prop-0 (\%) & APO (all) & APO (non-0) & Prop-0 (\%) & APO (all) & APO (non-0) &  \\ \midrule
Uniform-100     & 75.65\%     & 0.15      & 1.11        & 75.29\%     & 0.15      & 1.11        & 76.81\%     & 0.15      & 1.12        & 79.48\%     & 0.14      & 1.17        &  \\
Uniform-1000    & 62.51\%     & 0.64      & 1.98        & 64.32\%     & 0.62      & 1.92        & 65.67\%     & 0.52      & 1.75        & 67.70\%     & 0.52      & 1.76        &  \\
Uniform-5000    & 24.91\%     & 1.36      & 2.97        & 26.39\%     & 1.33      & 3.05        & 24.54\%     & 1.50      & 3.03        & 26.64\%     & 1.23      & 2.74        &  \\
Uniform-10000   & -           & -         & -           & -           & -         & -           & -           & -         & -           & -           & -         & -           &  \\
Cluster-100     & 70.66\%     & 0.23      & 1.26        & 70.13\%     & 0.24      & 1.28        & 72.34\%     & 0.24      & 1.30        & 75.35\%     & 0.20      & 1.28        &  \\
Cluster-1000    & 59.13\%     & 0.85      & 2.30        & 58.29\%     & 0.80      & 2.19        & 61.18\%     & 0.77      & 2.17        & 62.52\%     & 0.66      & 1.97        &  \\
Cluster-5000    & 23.65\%     & 1.93      & 3.88        & 23\%        & 1.64      & 3.33        & 23.58\%     & 1.64      & 3.31        & 24.76\%     & 1.48      & 3.05        &  \\
Cluster-10000   & -           & -         & -           & -           & -         & -           & -           & -         & -           & -           & -         & -           &  \\
Explosion-100   & 73.72\%     & 0.17      & 1.15        & 74.35\%     & 0.17      & 1.14        & 75.24\%     & 0.17      & 1.17        & 77.77\%     & 0.17      & 1.22        &  \\
Explosion-1000  & 62.92\%     & 0.69      & 2.03        & 61.35\%     & 0.67      & 2.03        & 63.32\%     & 0.62      & 1.89        & 65.58\%     & 0.57      & 1.84        &  \\
Explosion-5000  & 25.09\%     & 1.53      & 3.27        & 23.99\%     & 1.41      & 3.06        & 24.20\%     & 1.55      & 3.13        & 26.04\%     & 1.30      & 2.82        &  \\
Explosion-10000 & -           & -         & -           & -           & -         & -           & -           & -         & -           & -           & -         & -           &  \\
Implosion-100   & 72.97\%     & 0.19      & 1.20        & 73.50\%     & 0.19      & 1.19        & 74.80\%     & 0.19      & 1.23        & 77.27\%     & 0.18      & 1.26        &  \\
Implosion-1000  & 61.52\%     & 0.77      & 2.22        & 60.27\%     & 0.77      & 2.20        & 64.14\%     & 0.65      & 2.01        & 65.01\%     & 0.62      & 1.95        &  \\
Implosion-5000  & 24.05\%     & 1.93      & 3.83        & 23.70\%     & 1.69      & 3.41        & 24.12\%     & 1.65      & 3.33        & 25.30\%     & 1.46      & 3.07        &  \\
Implosion-10000 & -           & -         & -           & -           & -         & -           & -           & -         & -           & -           & -         & -           &  \\ \bottomrule
                & \multicolumn{6}{c}{PF-50}                                                     & \multicolumn{6}{c}{PF-100}                                                    &  \\
                & \multicolumn{3}{c}{Vanilla}           & \multicolumn{3}{c}{PUPO}              & \multicolumn{3}{c}{Vanilla}           & \multicolumn{3}{c}{PUPO}              &  \\
                & Prop-0 (\%) & APO (all) & APO (non-0) & Prop-0 (\%) & APO (all) & APO (non-0) & Prop-0 (\%) & APO (all) & APO (non-0) & Prop-0 (\%) & APO (all) & APO (non-0) &  \\ \midrule
Uniform-100     & 77.46\%     & 0.14      & 1.09        & 77.57\%     & 0.15      & 1.10        & 80.07\%     & 0.12      & 1.08        & 80.31\%     & 0.12      & 1.08        &  \\
Uniform-1000    & 59.68\%     & 0.74      & 2.00        & 60.73\%     & 0.70      & 1.94        & 65.66\%     & 0.56      & 1.78        & 66.06\%     & 0.54      & 1.75        &  \\
Uniform-5000    & 18.66\%     & 2.36      & 3.92        & 20.04\%     & 2.14      & 3.65        & 23.09\%     & 1.56      & 2.94        & 23.60\%     & 1.49      & 2.92        &  \\
Uniform-10000   & -           & -         & -           & -           & -         & -           & -           & -         & -           & -           & -         & -           &  \\
Cluster-100     & 71.41\%     & 0.24      & 1.27        & 71.53\%     & 0.25      & 1.29        & 72.49\%     & 0.24      & 1.29        & 73.64\%     & 0.23      & 1.28        &  \\
Cluster-1000    & 52.39\%     & 1.21      & 2.67        & 52.53\%     & 1.17      & 2.60        & 54.38\%     & 1.09      & 2.54        & 56.51\%     & 0.97      & 2.37        &  \\
Cluster-5000    & 15.08\%     & 4.52      & 6.54        & 15.42\%     & 4.06      & 5.95        & 16.08\%     & 4.20      & 6.31        & 18\%        & 3.13      & 4.93        &  \\
Cluster-10000   & -           & -         & -           & -           & -         & -           & -           & -         & -           & -           & -         & -           &  \\
Explosion-100   & 74.92\%     & 0.19      & 1.16        & 74.82\%     & 0.18      & 1.17        & 76.63\%     & 0.17      & 1.17        & 77.34\%     & 0.16      & 1.16        &  \\
Explosion-1000  & 55.35\%     & 0.98      & 2.33        & 55.95\%     & 0.98      & 2.34        & 57.53\%     & 0.89      & 2.21        & 59.66\%     & 0.81      & 2.12        &  \\
Explosion-5000  & 16.05\%     & 3.52      & 5.18        & 16.64\%     & 3.39      & 5.10        & 17.60\%     & 3.01      & 4.62        & 18.77\%     & 2.73      & 4.40        &  \\
Explosion-10000 & -           & -         & -           & -           & -         & -           & -           & -         & -           & -           & -         & -           &  \\
Implosion-100   & 73.76\%     & 0.24      & 1.30        & 73.68\%     & 0.24      & 1.28        & 75.67\%     & 0.22      & 1.29        & 76.24\%     & 0.22      & 1.29        &  \\
Implosion-1000  & 54.17\%     & 1.40      & 3.09        & 55.04\%     & 1.42      & 3.15        & 59.29\%     & 1.14      & 2.79        & 60.07\%     & 1.13      & 2.81        &  \\
Implosion-5000  & 15.18\%     & 8.14      & 10.81       & 15.90\%     & 8.46      & 11.59       & 17.79\%     & 7.46      & 10.38       & 18.88\%     & 6.02      & 8.88        &  \\
Implosion-10000 & -           & -         & -           & -           & -         & -           & -           & -         & -           & -           & -         & -           &  \\ \bottomrule
                & \multicolumn{6}{c}{INViT-50}                                                  & \multicolumn{6}{c}{INViT-100}                                                 &  \\
                & \multicolumn{3}{c}{Vanilla}           & \multicolumn{3}{c}{PUPO}              & \multicolumn{3}{c}{Vanilla}           & \multicolumn{3}{c}{PUPO}              &  \\
                & Prop-0 (\%) & APO (all) & APO (non-0) & Prop-0 (\%) & APO (all) & APO (non-0) & Prop-0 (\%) & APO (all) & APO (non-0) & Prop-0 (\%) & APO (all) & APO (non-0) &  \\ \midrule
Uniform-100     & 81.52\%     & 0.10      & 1.05        & 0.82        & 0.10      & 1.04        & 80.95\%     & 0.11      & 1.07        & 0.82        & 0.10      & 1.05        &  \\
Uniform-1000    & 86\%        & 0.16      & 1.37        & 0.86        & 0.15      & 1.33        & 85.80\%     & 0.17      & 1.43        & 0.87        & 0.14      & 1.33        &  \\
Uniform-5000    & 43.52\%     & 0.18      & 1.56        & 0.44        & 0.17      & 1.05        & 43.38\%     & 0.20      & 1.61        & 0.44        & 0.17      & 1.48        &  \\
Uniform-10000   & 87.39\%     & 0.18      & 1.55        & 0.88        & 0.17      & 1.50        & 87.32\%     & 0.20      & 1.65        & 0.88        & 0.16      & 1.47        &  \\
Cluster-100     & 79.78\%     & 0.14      & 1.18        & 0.80        & 0.13      & 1.17        & 79.26\%     & 0.16      & 1.26        & 0.80        & 0.14      & 1.19        &  \\
Cluster-1000    & 86.02\%     & 0.22      & 1.88        & 0.86        & 0.19      & 1.63        & 85.56\%     & 0.23      & 1.83        & 0.87        & 0.18      & 1.61        &  \\
Cluster-5000    & 43.72\%     & 0.20      & 1.61        & 0.44        & 0.18      & 1.55        & 43.32\%     & 0.27      & 2.20        & 0.44        & 0.17      & 1.52        &  \\
Cluster-10000   & 87.40\%     & 0.84      & 6.91        & 0.88        & 0.88      & 7.34        & 87.15\%     & 0.71      & 5.74        & 0.88        & 0.87      & 7.77        &  \\
Explosion-100   & 80.56\%     & 0.12      & 1.11        & 0.81        & 0.12      & 1.11        & 80.09\%     & 0.13      & 1.15        & 0.81        & 0.12      & 1.11        &  \\
Explosion-1000  & 85.42\%     & 0.30      & 2.34        & 0.86        & 0.28      & 2.23        & 85.09\%     & 0.30      & 2.33        & 0.86        & 0.26      & 2.16        &  \\
Explosion-5000  & 43.54\%     & 0.42      & 3.47        & 0.44        & 0.48      & 4.07        & 43.40\%     & 0.53      & 4.33        & 0.44        & 0.43      & 3.73        &  \\
Explosion-10000 & 87\%        & 0.95      & 7.68        & 0.87        & 1.03      & 8.38        & 87\%        & 0.97      & 7.58        & 0.88        & 0.88      & 7.57        &  \\
Implosion-100   & 79.98\%     & 0.12      & 1.12        & 0.80        & 0.12      & 1.12        & 79.77\%     & 0.14      & 1.17        & 0.80        & 0.13      & 1.13        &  \\
Implosion-1000  & 85.69\%     & 0.19      & 1.54        & 0.86        & 0.18      & 1.54        & 85.39\%     & 0.21      & 1.64        & 0.86        & 0.18      & 1.55        &  \\
Implosion-5000  & 43.39\%     & 0.23      & 1.88        & 0.44        & 0.22      & 1.81        & 43.33\%     & 0.24      & 1.92        & 0.44        & 0.19      & 1.66        &  \\
Implosion-10000 & 87.34\%     & 0.30      & 2.47        & 0.88        & 0.29      & 2.39        & 87.16\%     & 0.27      & 2.18        & 0.88        & 0.30      & 2.62        & \\
\bottomrule 
\end{tabular}%
}
\vskip -0.1in
\end{table}

% Table \ref{detailed_norm} demonstrate the detailed results of the total norm of network with different training methods. We can see that PUPO training reduces the sum of the total norms of the model parameters, acting as an implicit regularization mechanism. 
% \begin{table}[]
% \centering
% \caption{The sum of Frobenius norm of the
% parameter matrices for each model after different training.}\label{detailed_norm}
% \vskip 0.15in
% \begin{tabular}{c|cccc}
% \toprule
%         & AM-50 & AM-100 & INViT-50 & INViT-100 \\ \midrule
% Vanilla & 63.07 & 64.07  & 79.21    & 79.10     \\
% PUPO    & 60.22 & 62.82  & 79.15    & 78.89 \\ \bottomrule    
% \end{tabular}%
% \vskip -0.1in
% \end{table}

\end{document}